\documentclass[10pt,journal,final,finalsubmission,twocolumn]{IEEEtran} 
\usepackage{amsmath,amssymb}
\usepackage{amsthm}
\usepackage{subfigure,graphicx}
\usepackage{algorithm,algorithmic}
\usepackage{booktabs}
\usepackage{booktabs}
\usepackage{multirow}
\usepackage{array}
\newtheorem{mylemma}{Lemma}
\newtheorem{mytheorem}{Theorem}

\usepackage{times,balance}
\usepackage{epsfig,epstopdf}
\usepackage{xcolor}
\newcommand{\mc}[3]{\multicolumn{#1}{#2}{#3}}
\newcommand{\mr}[3]{\multirow{#1}{#2}{#3}}

\begin{document}
%
\title{Correlated Logistic Model with Elastic Net Regularization for Multilabel Image Classification}
%
%
\author{Qiang~Li,
Bo~Xie,
Jane~You,
Wei~Bian,~\IEEEmembership{Member,~IEEE,}
and Dacheng~Tao,~\IEEEmembership{Fellow,~IEEE}
\thanks{}
\thanks{This research was supported in part by Australian Research Council Projects FT-130101457, DP-140102164 and LE-140100061. The funding support from the Hong Kong government under its General Research Fund (GRF) scheme (Ref. no. 152202/14E) and the Hong Kong Polytechnic University Central Research Grant is greatly appreciated.}
\thanks{Q. Li is with the Centre for Quantum Computation and Intelligent Systems, Faculty of Engineering and Information Technology, University of Technology Sydney, 81 Broadway, Ultimo, NSW 2007, Australia, and also with Department of Computing, The Hong Kong Polytechnic University, Hung Hom, Kowloon, Hong Kong (e-mail: leetsiang.cloud@gmail.com).}
\thanks{W. Bian and D. Tao are with the Centre for Quantum Computation and Intelligent Systems, Faculty of Engineering and Information Technology, University of Technology Sydney, 81 Broadway, Ultimo, NSW 2007, Australia (e-mail: wei.bian@uts.edu.au, dacheng.tao@uts.edu.au).}
\thanks{B. Xie is with College of Computing, Georgia Institute of Technology, Atlanta, GA 30345, USA (email: zixu1986@gmail.com).}
\thanks{J. You is with Department of Computing, The Hong Kong Polytechnic University, Hung Hom, Kowloon, Hong Kong (e-mail: csyjia@comp.polyu.edu.hk).}
\thanks{\textcircled{c}20XX IEEE. Personal use of this material is permitted. Permission from IEEE must be obtained for all other uses, in any current or future media, including reprinting/republishing this material for advertising or promotional purposes, creating new collective works, for resale or redistribution to servers or lists, or reuse of any copyrighted component of this work in other works.}
}


\maketitle

\begin{abstract}
In this paper, we present correlated logistic model (CorrLog) for multilabel image classification. CorrLog extends conventional Logistic Regression model into multilabel cases, via explicitly
modelling the pairwise correlation between labels. In addition, we propose to learn model parameters of CorrLog with Elastic Net regularization, which helps exploit the sparsity in feature selection and label correlations and thus further boost the performance of multilabel classification. CorrLog can be efficiently learned, though approximately, by regularized maximum pseudo
likelihood estimation (MPLE), and it enjoys a satisfying generalization bound
that is independent of the number of labels. CorrLog performs competitively for multilabel image classification on benchmark datasets MULAN scene, MIT outdoor scene, PASCAL VOC 2007 and PASCAL VOC 2012, compared to the state-of-the-art multilabel classification algorithms.
\end{abstract}


\IEEEpeerreviewmaketitle

\section{Introduction} Multilabel classification (MLC) extends conventional single label
classification (SLC) by allowing an instance to be
assigned to multiple labels from a label set. It occurs naturally
from a wide range of practical problems, such as document
categorization, image classification, music annotation, webpage classification and bioinformatics applications, where each instance can be simultaneously described by several class labels out of a candidate label set.
MLC is also closely related to many other research areas, such as subspace learning \cite{bian2014asymptotic}, nonnegative matrix factorization \cite{liu2016perf}, multi-view learning \cite{xu2015multi} and multi-task learning \cite{liu2016algorithm}.
Because of its great generality and wide applications, MLC has received
increasing attentions in recent years from machine learning, data mining, to computer vision communities, and developed rapidly with both algorithmic and theoretical achievements
\cite{MLKNN,IBLR,HsuKLZ09,MLreview1,MLreview2,ReverseMLC}.

The key feature of MLC that makes it distinct from SLC is label correlation, without which classifiers can be trained independently for each individual label and MLC degenerates to SLC. The correlation between different labels can be verified by calculating the statistics, e.g., $\chi^2$ test and Pearson's correlation coefficient, of their distributions. According to \cite{LabelDependence}, there are two types of label correlations (or dependence), i.e., the conditional correlations and the unconditional correlations, wherein the former describes the label correlations conditioned on a given instance while the latter summarizes the global label correlations of only label distribution by marginalizing out the instance. From a classification point of view,
modelling of label conditional correlations is preferable
since they are directly related to prediction; however, proper utilization of unconditional correlations is also helpful, but in an average sense because of the marginalization. Accordingly, quite a number of MLC algorithms have been proposed in the past a few years, by exploiting either of the two types of label correlations,\footnote{Studies on MLC, from different perspectives rather than label correlations, also exit in the literature, e.g., by defining different loss functions, dimension reduction and classifier ensemble methods, but are not in the scope of this paper. } and below, we give a brief review of the representative ones. As it is a very big literature, we cannot cover all the algorithms. The recent surveys \cite{MLreview1,MLreview2}
contain many references omitted from this paper.
\begin{itemize}

\item \textit{By exploiting unconditional label correlations:}
A large class of MLC algorithms that utilize unconditional label correlations are built upon label transformation. The key idea is to find new representation for the label vector (one dimension corresponds to an individual label), so that the transformed labels or responses are uncorrelated and thus can be predicted independently. Original label vector needs to be recovered after the prediction. MLC algorithms using label transformation include \cite{YuYTK06} which utilizes low-dimensional embedding and \cite{HsuKLZ09} and \cite{Zhou:2012} which use random projections. Another strategy of using unconditional label correlations, e.g., used in the stacking method \cite{IBLR} and the ``Curds'' \& ``Whey'' procedure \cite{CandW}, is first to predict each individual label independently and correct/adjust the prediction by proper post-processing. Algorithms are also proposed based on
co-occurrence or structure information extracted from the label set,
which include random $k$-label sets (RAKEL) \cite{Rakel},
pruned problem transformation (PPT) \cite{PPT_MLC}, hierarchical
binary relevance (HBR) \cite{HBR} and hierarchy of multilabel
classifiers (HOMER) \cite{MLreview1}. Regression-based models, including reduced-rank regression and multitask learning, can also be used for MLC, with an interpretation of utilizing unconditional label correlations \cite{LabelDependence}.

\item \textit{By exploiting conditional label correlations:}
MLC algorithms in this category are diverse and often developed by specific heuristics.
For example, multilabel $K$-nearest neighbour
(MLkNN) \cite{MLKNN} extends KNN to the multilabel situation, which applies maximum a posterior (MAP) label prediction by obtaining the prior label distribution within the $K$ nearest neighbours of an instance. Instance-based logistic regression (IBLR)
\cite{IBLR} is also a localized algorithm, which modifies logistic regression by using label information from the neighbourhood as features. Classifier chain (CC)
\cite{ClassifierChain}, as well as its ensemble and probabilistic
variants \cite{PCC}, incorporate label correlations into a chain of
binary classifiers, where the prediction of a label uses previous
labels as features.
Channel coding based MLC techniques such as principal label space transformation (PLST) \cite{tai2012multilabel} and maximum margin output coding (MMOC) \cite{zhang2012maximum} proposed to select codes that exploits conditional label correlations.
Graphical models, e.g., conditional random fields (CRFs) \cite{CRF_Lafferty2001}, are also applied to MLC, which provides a richer framework to handle conditional label correlations.

\end{itemize}

\subsection{Multilabel Image Classification}
Multilabel image classification
belongs to the generic scope of MLC, but handles the specific problem of
predicting the presence or absence of multiple object categories in an image.
Like many related high-level vision tasks such as object recognition \cite{belongie2002shape,wang2013improved}, visual tracking \cite{wang2015video}, image annotation \cite{feng2004multiple,fan2011structured,mao2013objective} and scene classification \cite{boutell2004learning,song2010biologically,zuo2014learning},
multilabel image classification \cite{luo2013multiview,luo2013manifold,sun2014multi,zhang2014multilabel,luo2015multiview,xu2016local} is very challenging due to large intra-class variation.
In general, the variation is caused by viewpoint, scale, occlusion, illumination, semantic context, etc.

On the one hand, many effective image representation schemes have been developed to handle this high-level vision task.
Most of the classical approaches derive from handcrafted image features, such as GIST \cite{oliva2001modeling}, dense SIFT \cite{bosch2007image}, VLAD \cite{jegou2010aggregating}, and object bank \cite{li2010object}.
In contrast, the very recent deep learning techniques have also been developed for image feature learning, such as deep CNN features \cite{krizhevsky2012imagenet, Chatfield14}. These techniques are more powerful than classical methods when learning from a very large amount of unlabeled images.

On the other hand, label correlations have also been exploited to significantly improve image classification performance.
Most of the current multilabel image classification algorithms are motivated by considering label correlations conditioned on image features, thus intrinsically falls into the CRFs framework.
For example, probabilistic label enhancement model (PLEM) \cite{li2014multi} designed to exploit image label co-occurrence pairs based on a maximum spanning tree construction
and a piecewise procedure is utilized to train the pairwise CRFs model.
More recently, clique generating machine (CGM) \cite{tan2015learning} proposed to learn the image label graph structure and parameters by iteratively activating a set of cliques. It also belongs to the CRFs framework, but the labels are not constrained to be all connected which may result in isolated cliques.

\subsection{Motivation and Organization}
Correlated logistic model (CorrLog) provides a more principled way to handle conditional label correlations, and enjoys several favourable properties: 1) built upon independent logistic regressions (ILRs), it offers an explicit way to model the pairwise (second order) label correlations; 2) by using the pseudo likelihood technique, the parameters of CorrLog can be learned approximately with a computational complexity linear with respect to label number; 3) the learning of CorrLog is stable,
and the empirically learned model enjoys a generalization error bound that is independent of label number. In addition, the results presented in this paper extend our previous study \cite{bian2012corrlog} in following aspects: 1) we introduce elastic net regularization to CorrLog, which facilitates the utilization of the sparsity in both feature selection and label correlations; 2) a learning algorithm for CorrLog based on soft thresholding is derived to handle the nonsmoothness of the elastic net regularization; 3) the proof of generalization bound is also extended for the new regularization;
4) we apply CorrLog to multilabel image classification, and achieve competitive results with the state-of-the-art methods of this area.

To ease the presentation, we first summarize the important notations in Table \ref{tab:notation}.
The rest of this paper is organized as follows. Section II introduces the model CorrLog with elastic net regularization. Section III presents algorithms for learning CorrLog by regularized maximum pseudo likelihood estimation, and for prediction with CorrLog by message passing. A generalization analysis of CorrLog based on the concept of algorithm stability is presented in Section IV. Section V to Section VII report results of empirical evaluations,
including experiments on synthetic dataset and on benchmark multilabel image classification datasets.

\newcolumntype{C}[1]{>{\centering\arraybackslash}m{#1}}
\begin{table}
\renewcommand{\arraystretch}{1.4}
\caption{Summary of important notations throughout this paper.}
\begin{center}
\begin{tabular}{C{2.4cm}||m{5.2cm}}
\hline\hline

Notation & Description \\
\hline
$\mathcal D = \{\mathbf x^{(l)},\mathbf y^{(l)}\}$ & training dataset with $n$ examples, $1\le l\le n$\\
\hline
$\mathcal D^k$ & modified training data set by replacing the $k$-th example of $\mathcal
D$ with an independent example\\
\hline
$\mathcal D^{\backslash k}$ & modified training data set by discarding the $k$-th example of $\mathcal D$\\
\hline
$\widetilde{\mathcal L} (\Theta)$& negative log pseudo likelihood over training dataset $\mathcal D^k$\\
\hline
$\widetilde{\mathcal L}_r(\Theta)$& regularized negative log pseudo likelihood over training dataset $\mathcal D^{\backslash k}$\\
\hline
$R_{en}(\Theta;\lambda_1,\lambda_2,\epsilon) $ & elastic net regularization with weights $\lambda_1$, $\lambda_2$ and parameter $\epsilon$\\
\hline
$\Theta=\{\boldsymbol\beta,\boldsymbol\alpha\}$ & model parameters of CorrLog\\
\hline
$\widetilde\Theta=\{\widetilde{\boldsymbol\beta},\widetilde{\boldsymbol\alpha}\} $ & empirical learned model parameters by maximum pseudo likelihood estimation over $\mathcal D$\\
\hline
$\widetilde\Theta^k=\{\widetilde{\boldsymbol\beta}^k,\widetilde{\boldsymbol\alpha}^k\}$ & empirical learned model parameters over $\mathcal D^k$ \\
\hline
$ \widetilde\Theta^{\backslash k}=\{\widetilde{\boldsymbol\beta}^{\backslash k},\widetilde{\boldsymbol\alpha}^{\backslash k}\}$ & empirical learned model parameters over $\mathcal D^{\backslash k}$\\
\hline
$\widetilde{\mathcal R}(\widetilde\Theta)$ & empirical error of the empirical model $\widetilde\Theta$ over training set $\mathcal D$\\
\hline
$\mathcal R(\widetilde\Theta)$ & generalization error of the empirical model $\widetilde\Theta$\\
\hline\hline
\end{tabular}
\end{center}
\label{tab:notation}
\end{table}

\section{Correlated Logistic Model}
We study the problem of learning a joint prediction $\mathbf
y=d(\mathbf x):\mathcal X \mapsto \mathcal Y$, where the instance
space $\mathcal X=\{\mathbf x:\|\mathbf x\|\le 1, \mathbf
x\in\mathbb R^D\}$ and the label space $\mathcal Y = \{-1,1\}^m$. By
assuming the conditional independence among labels, we can model MLC
by a set of independent logistic regressions (ILRs). Specifically, the conditional probability $p_{lr}(\mathbf y|\mathbf x)$ of ILRs is given by
\begin{equation}\label{eq:inde_prod}
\begin{aligned}
p_{lr}(\mathbf y|\mathbf x) &= \prod_{i=1}^mp_{lr}(\mathbf
y_i|\mathbf x) \\
&= \prod_{i=1}^m\frac{\exp\left(\mathbf y_i
\beta_i^T\mathbf x\right)}{\exp\left(\mathbf \beta_i^T\mathbf
x\right)+\exp\left(-\mathbf \beta_i^T\mathbf x\right)},
\end{aligned}
\end{equation}
where $\beta_i\in\mathbb R^D$ is the coefficients for the $i$-th
logistic regression (LR) in ILRs. For the convenience of expression, the bias of
the standard LR is omitted here, which is equivalent to augmenting the feature of
$\mathbf x$ with a constant.

Clearly, ILRs (\ref{eq:inde_prod}) enjoys several merits, such as, it can be
learned efficiently, in particular with a linear computational complexity with respect to label number $m$, and its probabilistic formulation inherently helps deal with the imbalance of positive and negative examples for each label, which is a common problem encountered by MLC. However, it ignores entirely the potential correlation among labels and thus tends to under-fit the true posterior $p_0(\mathbf
y|\mathbf x)$, especially when the label number $m$ is large.

\subsection{Correlated Logistic Regressions}
CorrLog tries to extend ILRs with as small effort as possible, so that the correlation among labels is explicitly modelled while the advantages of ILRs can be also preserved. To achieve this, we propose to
augment (\ref{eq:inde_prod}) with a simple function $q(\mathbf y)$
and reformulate the posterior probability as
\begin{equation}\label{eq:augment}
p(\mathbf y|\mathbf x) \propto p_{lr}(\mathbf y|\mathbf x)q(\mathbf
y).
\end{equation}
As long as $q(\mathbf y)$ cannot be decomposed into independent product terms for individual labels, it introduces label correlations into $p(\mathbf y|\mathbf x)$. It is worth noticing that we assumed $q(\mathbf y)$ to be independent of $\mathbf x$. Therefore, (\ref{eq:augment}) models label correlations in an average
sense. This is similar to the concept of ``marginal correlations'' in
MLC \cite{LabelDependence}. However, they are intrinsically
different, because (\ref{eq:augment}) integrate the correlation into the posterior
probability, which directly aims at prediction. In addition, the
idea used in (\ref{eq:augment}) for correlation modelling is also
distinct from the ``Curds and Whey'' procedure in \cite{CandW} which
corrects outputs of multivariate linear regression by reconsidering
their correlations to the true responses.

In this paper, we choose $q(\mathbf y)$ to be the following quadratic form,
\begin{equation}\label{eq:q_y}
q(\mathbf y) = \exp\left\{\sum_{i<j}\alpha_{ij}\mathbf y_i\mathbf
y_j\right\}.
\end{equation}
It means that $\mathbf y_i$ and $\mathbf y_j$ are positively correlated given
$\alpha_{ij}>0$ and negatively correlated given $\alpha_{ij}<0$. It is also
possible to define $\alpha_{ij}$ as functions of $\mathbf x$, but
this will drastically increase the number of model parameters, e.g., by
$\mathcal O(m^2D)$ if linear functions are used.

By substituting (\ref{eq:q_y}) into (\ref{eq:augment}), we
obtain the conditional probability for CorrLog
\begin{equation}\label{eq:alg}
p(\mathbf y|\mathbf x;\Theta)\propto\exp\left\{\sum_{i=1}^m\mathbf
y_i\beta_i^T\mathbf x + \sum_{i<j}\alpha_{ij}\mathbf y_i\mathbf
y_j\right\},
\end{equation}
where the model parameter $\Theta=\{\boldsymbol\beta,\boldsymbol\alpha\}$ contains
$\boldsymbol\beta=[\beta_1,...,\beta_m]$ and $\boldsymbol\alpha =
[\alpha_{12},...,\alpha_{(m-1)m}]^T$. It can be seen that CorrLog is a simple
modification of (\ref{eq:inde_prod}), by using a quadratic term to
adjust the joint prediction, so that hidden label correlations can be exploited. In addition,
CorrLog is closely related to popular statistical models for joint
modelling of binary variables. For example, conditional on $\mathbf x$, (\ref{eq:alg}) is exactly an
Ising model \cite{IsingModel_SparseLogReg} for $\mathbf y$. It can also be treated as a special instance
of CRFs \cite{CRF_Lafferty2001}, by
defining features $\phi_i(\mathbf x,\mathbf y)=\mathbf y_i\mathbf x$
and $\psi_{ij}(\mathbf y)=\mathbf y_i\mathbf y_j$. Moreover,
classical model multivariate probit (MP) \cite{MultivariteProbit}
also models pairwise correlations in $\mathbf y$. However, it
utilizes Gaussian latent variables for correlation modelling, which
is essentially different from CorrLog.

\subsection{Elastic Net Regularization}
Given a set of training data $\mathcal D = \{\mathbf x^{(l)},\mathbf y^{(l)}: 1\le l\le n\}$, CorrLog can be learned by regularized maximum log likelihood estimation (MLE), i.e.,
\begin{equation}\label{eq:MLE}
\widehat\Theta=\arg\min_{\Theta}{\mathcal L}(\Theta)+ R (\Theta),
\end{equation}
where $\mathcal L(\Theta)$ is the negative log likelihood
\begin{equation}\label{eq:LogLikelihood}
{\mathcal L}(\Theta) = -\frac1n\sum_{l=1}^n\log
p(\mathbf y^{(l)}|\mathbf x^{(l)};\Theta),
\end{equation}
and $ R(\Theta)$ is a properly chosen regularization.

A possible choice for $ R(\Theta)$ is the $\ell_2$ regularizer,
\begin{align}
R_2(\Theta;\lambda_1,\lambda_2)=\lambda_1 \sum_{i=1}^m\|\beta_i\|_2^2 + \lambda_2\sum_{i<j}|\alpha_{ij}|^2,
\end{align}
with $\lambda_1$, $\lambda_2>0$ being the weighting parameters. The $\ell_2$ regularization enjoys the merits of computational flexibility and learning stability. However, it is unable to exploit any sparsity that can be possessed by the problem at hand. For example, for MLC, it is likely that the prediction of each label $\mathbf y_i$ only depends on a subset of the $D$ features of $\mathbf x$, which implies the sparsity of $\beta_i$. Besides, $\boldsymbol\alpha$ can also be sparse since not all labels in $\mathbf y$ are correlated to each other. $\ell_1$ regularizer is another choice for $\mathcal R(\Theta)$, especially regarding model sparsity. Nevertheless, it has been noticed by several studies that $\ell_1$ regularized algorithms are inherently unstable, that is, a slight change of the training data set can lead to substantially different prediction models. Based on above consideration, we propose to use the elastic net regularizer \cite{zou05regularization}, which is a combination of $\ell_2$ and $\ell_1$ regularizers and inherits their individual advantages, i.e., learning stability and model sparsity,
\begin{align}
R_{en}(\Theta;\lambda_1,\lambda_2,\epsilon) &= \lambda_1 \sum_{i=1}^m(\|\beta_i\|_2^2 + \epsilon\|\beta_i\|_1) \notag\\
&+ \lambda_2\sum_{i<j}(|\alpha_{ij}|^2 + \epsilon|\alpha_{ij}|),
\end{align}
where $\epsilon\ge0$ controls the trade-off between the $\ell_1$ regularization and the $\ell_2$ regularization, and large $\epsilon$ encourages a high level of sparsity.

\section{Algorithms}
In this section, we derive algorithms for learning and prediction with CorrLog. The exponentially large size of the label space $\mathcal Y=\{-1,1\}^m$ makes exact algorithms for CorrLog computationally intractable, since the conditional probability (\ref{eq:alg}) needs to be normalized by the partition function
\begin{align}
A(\Theta) = \sum_{y\in\mathcal Y}\exp\left\{\sum_{i=1}^m\mathbf
y_i\beta_i^T\mathbf x + \sum_{i<j}\alpha_{ij}\mathbf y_i\mathbf
y_j\right\},
\end{align}
which is a summation over an exponential number of terms. Thus, we turn to approximate learning and prediction algorithms, by exploiting the pseudo likelihood and the message passing techniques.

\subsection{Approximate Learning via Pseudo Likelihood}
Maximum pseudo
likelihood estimation (MPLE) \cite{PseudoLikelihoodBesag75} provides an alternative approach for estimating model parameters, especially when the partition function of the likelihood cannot be evaluated efficiently.
It was developed in the field of spatial dependence analysis and has been widely applied to the estimation of various statistical models,
from the Ising model \cite{IsingModel_SparseLogReg} to the CRFs \cite{CRF_PseudoLikelihood}.
Here, we apply MPLE to the learning of parameter $\Theta$ in CorrLog.

The pseudo likelihood of the model over $m$ jointly distributed random variables is defined as the product of the conditional probability of each individual random variables conditioned on all the rest ones. For CorrLog (\ref{eq:alg}), its pseudo likelihood is given by
\begin{equation}\label{eq:cpl}
\widetilde p(\mathbf y|\mathbf x;\Theta)= \prod_{i=1}^mp(\mathbf
y_i|\mathbf y_{-i},\mathbf x;\Theta),
\end{equation}
where $\mathbf y_{-i}=[\mathbf y_1,...,\mathbf y_{i-1},\mathbf
y_{i+1},...,\mathbf y_{m}]$ and the conditional probability
$p(\mathbf y_i|\mathbf y_{-i},\mathbf x;\Theta)$ can be directly
obtained from (\ref{eq:alg}),
\begin{equation}\label{eq:cond_prob}
\begin{aligned}
&p(\mathbf y_i|\mathbf y_{-i},\mathbf x;\Theta)
=\\
&\frac1{1+\exp\left\{-2\mathbf y_i\left(\beta_i^T\mathbf x +
\sum_{j=i+1}^m\alpha_{ij}\mathbf y_j +
\sum_{j=1}^{i-1}\alpha_{ji}\mathbf y_j\right)\right\}}.
\end{aligned}
\end{equation}
Accordingly, the negative log pseudo likelihood over the training data $\mathcal D$ is given by
\begin{equation}\label{eq:L}
\widetilde{\mathcal L}(\Theta) = -\frac1n\sum_{l=1}^n\sum_{i=1}^m\log
p(\mathbf y_{i}^{(l)}|\mathbf y_{-i}^{(l)},\mathbf x^{(l)};\Theta).
\end{equation}
To this end, the optimal model parameter $\widetilde\Theta=\{\widetilde{\boldsymbol\beta},\widetilde{\boldsymbol\alpha}\}$ of CorrLog can be learned approximately by the elastic net regularized MPLE,
\begin{align}\label{eq:estimation}
\widetilde\Theta&= \arg\min_{\Theta}\widetilde{\mathcal L}_r(\Theta)\notag\\
&=\arg\min_{\Theta} \widetilde{\mathcal
L}(\Theta) + R_{en}(\Theta;\lambda_1,\lambda_2,\epsilon) .
\end{align}
where $\lambda_1$, $\lambda_2$ and $\epsilon$ are tuning parameters.

\noindent
\textbf{A First-Order Method by Soft Thresholding:}
Problem (\ref{eq:estimation}) is a convex optimization problem, thanks to the convexity of the logarithmic loss function and the elastic net regularization, and thus a unique optimal solution. However, the elastic net regularization is non-smooth due to the $\ell_1$ norm regularizer, which makes direct gradient based algorithm inapplicable. The main idea of our algorithm for solving (\ref{eq:estimation}) is to divide the objective function into smooth and non-smooth parts, and then apply the soft thresholding technique to deal with the non-smoothness.

Denoting by $J_s(\Theta)$ the smooth part of $\widetilde{\mathcal L}_r(\Theta)$, i.e.,
\begin{align}
J_s(\Theta) = \widetilde{\mathcal L}(\Theta) + \lambda_1 \sum_{i=1}^m\|\beta_i\|_2^2 + \lambda_2\sum_{i<j}|\alpha_{ij}|^2,
\end{align}
its gradient $\nabla{J_s}$ at the $k$-th iteration $\Theta^{(k)} = \{\boldsymbol\beta^{(k)},\boldsymbol{\alpha}^{(k)}\}$ is given by

\begin{equation}\label{eq:gradients}
\left\{\begin{array}{l}
\nabla{J_s}_{\beta_{i}} (\Theta^{(k)})=\frac1n\sum\limits_{l=1}^n\xi_{li}\mathbf x^{(l)} + 2\lambda_1\beta_{i}^{(k)}\\
\nabla{J_s}_{\alpha_{ij}} (\Theta^{(k)})=\frac1n\sum\limits_{l=1}^n\left(\xi_{li}\mathbf y_{j}^{(l)} +\xi_{lj}\mathbf y_i^{(l)}\right)+2\lambda_2\alpha_{ij}^{(k)}\\
\end{array} \right.
\end{equation}
with
\begingroup\makeatletter\def\f@size{9}\check@mathfonts 
\begin{equation}\label{eq:xi}
\begin{aligned}
&\xi_{li} = \\
&\frac{-2\mathbf y_i^{(l)}}{1+\exp\left\{2\mathbf
y_i^{(l)}\left(\beta_i^{(k)T}\mathbf x^{(l)} +
\sum_{j=i+1}^m\alpha_{ij}^{(k)}\mathbf y_j^{(l)} +
\sum_{j=1}^{i-1}\alpha_{ji}^{(k)}\mathbf y_j^{(l)}\right)\right\}}.
\end{aligned}
\end{equation}
\endgroup 
Then, a surrogate $J(\Theta)$ of the objective function $\widetilde{\mathcal L}_r(\Theta)$ in (\ref{eq:estimation}) can be obtained by using $\nabla{J_s}(\Theta^{(k)})$, i.e.,
\begingroup\makeatletter\def\f@size{9}\check@mathfonts 
\begin{align}\label{eq:surrogate}
&J(\Theta;\Theta^{(k)}) = J_s(\Theta^{(k)})\notag\\
&+\sum_{i=1}^m \langle \nabla{J_s}_{\beta_i}(\Theta^{(k)}), \beta_i - \beta_i^{(k)}\rangle +
\frac1{2\eta}\|\beta_i - \beta_i^{(k)}\|_2^2 + \lambda_1\epsilon\|\beta_i\|_1\notag\\
&+ \sum_{i<j} \langle \nabla{J_s}_{\alpha_{ij}}(\Theta^{(k)}), \alpha_{ij} - \alpha_{ij}^{(k)}\rangle +\frac1{2\eta}(\alpha_{ij} - \alpha_{ij}^{(k)})^2 + \lambda_2\epsilon|\alpha_{ij}|.\notag\\
\end{align}
\endgroup 
The parameter $\eta$ in (\ref{eq:surrogate}) servers a similar role to the variable updating step size in gradient descent methods, and it is set such that $1/\eta$ is larger than the Lipschitz constant of $\nabla{J_s}(\Theta^{(k)})$. For such $\eta$, it can be shown that $J(\Theta)\ge\widetilde{\mathcal L}_r(\Theta)$ and $J(\Theta^{(k)})=\widetilde{\mathcal L}_r(\Theta^{(k)})$. Therefore, the update of $\Theta$ can be realized by the minimization
\begin{align}
\Theta^{(k+1)} = \arg\min_{\Theta} J(\Theta;\Theta^{(k)}),
\end{align}
which is solved by the soft thresholding function $\mathcal S(\cdot)$, i.e.,
\begin{equation}\label{eq:soft-thresholding}
\left\{\begin{array}{l}
\beta_{i}^{(k+1)} = \mathcal S(\beta_i^{(k)} - \eta\nabla{J_{s\beta_i}}(\Theta^{(k)});\lambda_1\epsilon)\\
\alpha_{ij}^{(k+1)} = \mathcal S(\alpha_{ij}^{(k)} - \eta\nabla{J_{s\alpha_{ij}}}(\Theta^{(k)});\lambda_2\epsilon),
\end{array} \right.
\end{equation}
where
\begin{align}
\mathcal S(u;\rho) = \left\{\begin{array}{ll}
u-0.5\rho, & \mbox{if}~u>0.5\rho\\
u+0.5\rho, & \mbox{if}~u<-0.5\rho\\
0, & \mbox{otherwise.}\\
\end{array}\right.
\end{align}
Iteratively applying (\ref{eq:soft-thresholding}) until convergence provides a first-order method for solving (\ref{eq:estimation}). Algorithm 1 presents the pseudo code for this procedure.

\begin{algorithm}[tb]
\caption{Learning CorrLog by Maximum Pseudo Likelihood Estimation with Elastic Net Regularization}
\label{alg:example}
\begin{algorithmic}
\STATE {\bfseries Input:} Training data $\mathcal D$, initialization $\boldsymbol\beta^{(0)}=\mathbf 0$, $\boldsymbol\alpha^{(0)}=\mathbf
0$, and learning rate
$\eta$, where $1/\eta$ is set larger than the Lipschitz constant of
$\nabla J_s(\Theta)$ (\ref{eq:surrogate}).
\STATE {\bfseries Output:} Model parameters
$\widetilde{\Theta}=({\widetilde{\boldsymbol\beta}^{(t)}},\widetilde{\boldsymbol\alpha}^{(t)})$.
\REPEAT

\STATE Calculating the gradient of $J_S(\Theta)$ at ${\Theta}^{(k)}=({{\boldsymbol\beta}^{(k)}},{\boldsymbol\alpha}^{(k)})$ by using (\ref{eq:gradients});
\STATE Updating ${\Theta}^{(k+1)}=({{\boldsymbol\beta}^{(k+1)}},{\boldsymbol\alpha}^{(k+1)})$ by using soft thresholding (\ref{eq:soft-thresholding});
\STATE\quad $k = k +1$
\UNTIL{Converged}
\end{algorithmic}
\end{algorithm}

\noindent
\textbf{Remark 1} From the above derivation, especially equations (\ref{eq:gradients}) and (\ref{eq:soft-thresholding}), the computational complexity of our learning algorithm is linear with respect to the label number $m$. Therefore, learning CorrLog is no more expensive than learning $m$ independent logistic regressions, which makes CorrLog scalable to the case of large label numbers.

\noindent
\textbf{Remark 2} It is possible to further speed up the learning algorithm. In particular, Algorithm 1 can be modified to have the optimal convergence rate in the sense of Nemirovsky and Yudin \cite{NemirovskyYudin}, i.e., $\mathcal O(1/k^2)$ wherein $k$ is the number of iterations. However, its convergence is usually as slow as in standard gradient descent methods. Actually, we only need to replace the current variable $\Theta^{(k)}$ in the surrogate (\ref{eq:surrogate}) by a weighted combination of the variables from previous iterations. As such modification is a direct application of the fast iterative shrinkage thresholding, \cite{FISTA}, we do not present the details here but leave readers to the reference.

\subsection{Joint Prediction by Message Passing}
For MLC, as the labels are not independent in general, the prediction task is actually a joint maximum a posterior (MAP) estimation over $p(\mathbf y|\mathbf x)$. In the case of CorrLog, suppose the model parameter $\widetilde{\Theta}$ is learned by the regularized MPLE from the last subsection, the
prediction of $\widehat{\mathbf y}$ for a new instance $\mathbf
x$ can be obtained by
\begin{align}\label{eq:map}
\widehat{\mathbf y} &= \arg\max_{\mathbf y\in\mathcal Y}p(\mathbf
y|\mathbf x;\widetilde{\Theta})\notag\\
&=\arg\max_{\mathbf y\in\mathcal Y} \exp\left\{\sum_{i=1}^m\mathbf
y_i\widetilde{\beta}_i^{ T}\mathbf x +
\sum_{i<j}\widetilde\alpha_{ij}\mathbf y_i\mathbf y_j\right\}.
\end{align}
We use the belief propagation (BP) to solve (\ref{eq:map}) \cite{Bishop2006}. Specifically, we run the max-product algorithm with uniformly initialized messages and an early stopping criterion with 50 iterations. Since the graphical model defined by $\boldsymbol\alpha$ in (\ref{eq:map}) has loops, we cannot guarantee the convergence of the algorithm. However, we found that it works well on all experiments in this paper.

\section{Generalization Analysis}
An important issue in designing a machine learning algorithm is generalization, i.e., how the algorithm will perform on the test data compared to on the training data. In the section, we present a generalization analysis for CorrLog, by using the concept of algorithmic stability \cite{StabilityAndGeneralization}. Our analysis follows two steps. First, we show that the
learning of CorrLog by MPLE is stable, i.e., the learned model parameter $\widetilde\Theta$ does not vary much given a slight change of the training data set $\mathcal D$. Then, we prove that the generalization error of CorrLog can be bounded by the
empirical error, plus a term related to the stability but independent
of the label number $m$.

\subsection{The Stability of MPLE}
The stability of a learning algorithm indicates how much the learned
model changes according to a small change of the training data set.
Denote by $\mathcal D^k$ a modified training data set the same with $\mathcal
D$ but replacing the $k$-th training example $(\mathbf x^{(k)},\mathbf
y^{(k)})$ by another independent example $(\mathbf x',\mathbf y')$.
Suppose $\widetilde{\Theta}$ and $\widetilde{\Theta}^k$ are the model parameters learned by MPLE
(\ref{eq:estimation}) on $\mathcal D$ and $\mathcal D^k$,
respectively. We intend to show that the difference between these two models, defined as
\begin{equation}
\|\widetilde{\Theta}^k-\widetilde{\Theta}\|\triangleq
\sum_{i=1}^m\|\widetilde{\boldsymbol\beta}_i^{k}-\widetilde{\boldsymbol\beta}_i\|
+
\sum_{i<j}|\widetilde{\boldsymbol\alpha}_{ij}^{k}-\widetilde{\boldsymbol\alpha}_{ij}|,
\mbox{~}\forall~ 1\le k\le n,
\end{equation}
is bounded by an order of $\mathcal O(1/n)$, so that the learning is stable for large $n$.

First, we need the following auxiliary model
$\widetilde{\Theta}^{\backslash k}=\{\widetilde{\boldsymbol\beta}^{\backslash k},\widetilde{\boldsymbol\alpha}^{\backslash k}\}$ learned on $\mathcal D^{\backslash k}$, which is the same with $\mathcal D$ but without the $k$-th example
\begin{equation}\label{eq:auxiliary}
\widetilde\Theta^{\backslash k}= \arg\min_{\Theta}{\widetilde{\mathcal L}}^{\backslash k}(\Theta)+ \mathcal R_{en}(\Theta;\lambda_1,\lambda_2,\epsilon),
\end{equation}
where
\begin{equation}\label{eq:L}
\widetilde{\mathcal L}^{\backslash k}(\Theta) = -\frac1n\sum_{l\ne
k}\sum_{i=1}^m\log p(\mathbf y_{i}^{(l)}|\mathbf
y_{-i}^{(l)},\mathbf x^{(l)};\Theta).
\end{equation}

The following Lemma provides an upper bound of the difference $\widetilde{\mathcal L}_r(\widetilde\Theta^{\backslash
k})-\widetilde{\mathcal L}_r(\widetilde\Theta)$.

\begin{mylemma} Given $\widetilde{\mathcal L}_r(\cdot)$ and $\widetilde\Theta$ defined in (\ref{eq:estimation}), and $\widetilde\Theta^{\backslash k}$ defined in (\ref{eq:auxiliary}), it holds for $\forall 1\le k\le n$,
\begingroup\makeatletter\def\f@size{9}\check@mathfonts 
\begin{align}\label{eq:lemma1_1}
&\widetilde{\mathcal L}_r(\widetilde\Theta^{\backslash k}) -\widetilde{\mathcal
L}_r(\widetilde\Theta)\le \notag\\
& \frac1n\left(\sum_{i=1}^m\log p(\mathbf
y_{i}^{(k)}|\mathbf y_{-i}^{(k)},\mathbf
x^{(k)};\widetilde\Theta^{\backslash k})-\sum_{i=1}^m\log p(\mathbf
y_{i}^{(k)}|\mathbf y_{-i}^{(k)},\mathbf
x^{(k)};\widetilde\Theta)\right)
\end{align}
\endgroup 
\end{mylemma}

\noindent
\begin{proof}
Denote by RHS the righthand side of (\ref{eq:lemma1_1}), we have
\begin{align}\notag
\mathrm{RHS}=\left(\widetilde{\mathcal L}_r(\widetilde\Theta^{\backslash k})
-\widetilde{\mathcal L}_r^{\backslash k}(\widetilde\Theta^{\backslash k})\right)
-\left( \widetilde{\mathcal L}_r(\widetilde\Theta) -\widetilde{\mathcal L}_r^{\backslash
k}(\widetilde\Theta)\right).
\end{align}
Furthermore, the definition of
$\widetilde\Theta^{\backslash k}$ implies $\widetilde{\mathcal L}_r^{\backslash
k}(\widetilde\Theta^{\backslash k})\le \widetilde{\mathcal L}_r^{\backslash
k}(\widetilde\Theta)$. Combining these two we have
(\ref{eq:lemma1_1}). This completes the proof.
\end{proof}

Next, we show a lower bound of the difference $\widetilde{\mathcal L}_r(\widetilde\Theta^{\backslash
k})-\widetilde{\mathcal L}_r(\widetilde\Theta)$.

\begin{mylemma} Given $\widetilde{\mathcal L}_r(\cdot)$ and $\widetilde\Theta$ defined in (\ref{eq:estimation}), and $\widetilde\Theta^{\backslash k}$ defined in (\ref{eq:auxiliary}), it holds for $\forall 1\le k\le n$,
\begin{equation}\label{eq:lemma2_1}
\widetilde{\mathcal L}_r(\widetilde\Theta^{\backslash k}) -\widetilde{\mathcal L}_r(\widetilde\Theta)
\ge\lambda_1\|\widetilde{\boldsymbol\beta}^{\backslash k}-\widetilde{\boldsymbol\beta}\|^2 +
\lambda_2\|\widetilde{\boldsymbol\alpha}^{\backslash k}-\widetilde{\boldsymbol\alpha}\|^2.
\end{equation}
\end{mylemma}

\begin{proof}
We define the following function
\begin{align}\notag
f(\Theta) = \widetilde{\mathcal L}_r(\Theta)-\lambda_1\|\boldsymbol\beta-\widetilde{\boldsymbol\beta}\|^2-\lambda_2\|\boldsymbol\alpha-\widetilde{\boldsymbol\alpha}\|^2.
\end{align}
Then, for (\ref{eq:lemma2_1}), it is sufficient to show that $f(\widetilde{\Theta}^{\backslash
k})\ge f(\widetilde{\Theta})$. By using (\ref{eq:estimation}), we have
\begin{align}
f(\Theta) &= \widetilde{\mathcal L}(\Theta)+
2\lambda_1\sum_{i=1}^m\beta_i^T\widetilde{\boldsymbol\beta}_i+
2\lambda_2\sum_{i<j}{\alpha_{ij}\widetilde{\alpha}_{ij}} \nonumber\\
&+\lambda_1\epsilon\sum_{i=1}^m\|{\beta}_i\|_1+\lambda_2\epsilon\sum_{i<j}|\alpha_{ij}|.
\end{align}
It is straightforward to verify that $f(\Theta)$ and $\widetilde{\mathcal L}_r(\Theta)$ in (\ref{eq:estimation}) have the same subgradient at $\widetilde{\Theta}$, i.e.,
\begin{align}
\partial f(\widetilde\Theta) = \partial \widetilde{\mathcal L}_r(\widetilde\Theta).
\end{align}
Since $\widetilde{\Theta}$ minimizes $\widetilde{\mathcal L}_r(\Theta)$, we have $\mathbf 0 \in\partial \widetilde{\mathcal L}_r(\widetilde\Theta) $ and thus $\mathbf 0 \in\partial f(\widetilde\Theta) $, which implies $\widetilde{\Theta}$ also minimizes $f(\Theta)$. Therefore $f(\widetilde{\Theta})\le
f(\widetilde{\Theta}^{\backslash k})$.
\end{proof}

In addition, by checking the Lipschitz continuous property of $\log
p(\mathbf y_i|\mathbf y_{-i},\mathbf x;\Theta)$, we have the
following Lemma 3.
\begin{mylemma}\label{lemma:lem1} Given $\widetilde\Theta$ defined in (\ref{eq:estimation}) and $\widetilde\Theta^{\backslash k}$ defined in (\ref{eq:auxiliary}), it holds for $\forall~ (\mathbf x,
\mathbf y)\in\mathcal X\times\mathcal Y$ and $\forall 1\le k\le n$
\begin{align}\label{eq:lemma3_1}
&\big |\sum_{i=1}^m\log p(\mathbf y_{i}|\mathbf y_{-i},\mathbf
x;\widetilde\Theta) - \sum_{i=1}^m\log p(\mathbf y_{i}|\mathbf
y_{-i},\mathbf x;\widetilde\Theta^{\backslash k})\big|\notag\\
&\le
2\sum_{i=1}^m\|\widetilde{\beta}_i-\widetilde{\beta}^{\backslash k}_i\|
+ 4\sum_{i<j}|\widetilde{\alpha}_{ij}-\widetilde{\alpha}^{\backslash
k}_{ij}|.
\end{align}
\end{mylemma}

\noindent
\begin{proof}
First, we have
\begin{align}\notag
\|\partial \log p(\mathbf
y_{i}|\mathbf y_{-i},\mathbf x;\Theta)/\partial \beta_i\|\le
2\|\mathbf x\|\le2,
\end{align}
and
\begin{align}\notag
|\partial \log p(\mathbf y_{i}|\mathbf
y_{-i},\mathbf x;\Theta)/\partial \alpha_{ij}|\le 4|\mathbf
y_i\mathbf y_j|=4.
\end{align}
That is $\log p(\mathbf
y_{i}|\mathbf y_{-i},\mathbf x;\Theta)$ is Lipschitz continuous with respect to $\beta_i$ and $\alpha_{ij}$, with constant $2$ and $4$, respectively. Therefore, (\ref{eq:lemma3_1}) holds.
\end{proof}

By combining the above three Lemmas, we have the following Theorem 1
that shows the stability of CorrLog.
\begin{mytheorem}
Given model parameters $\widetilde\Theta=\{\widetilde{\boldsymbol\beta},\widetilde{\boldsymbol\alpha}\}$ and $\widetilde\Theta^{k}=\{\widetilde{\boldsymbol\beta}^{k},\widetilde{\boldsymbol\alpha}^{k}\}$ learned on training datasets $ \mathcal D$ and $\mathcal D^k$, respectively, both by (\ref{eq:estimation}), it holds that
\begin{equation}\label{eq:estimation_stability}
\sum_{i=1}^m\|\widetilde{\boldsymbol\beta}_i^{k}-\widetilde{\boldsymbol\beta}_i\|
+
\sum_{i<j}|\widetilde{\boldsymbol\alpha}_{ij}^{k}-\widetilde{\boldsymbol\alpha}_{ij}|\le
\frac{16}{\min(\lambda_1,\lambda_2)n}.
\end{equation}
\end{mytheorem}

\noindent
\begin{proof}
By combining (\ref{eq:lemma1_1}),
(\ref{eq:lemma2_1}) and (\ref{eq:lemma3_1}), we have
\begin{align}
&\|\widetilde{\boldsymbol\beta}^{\backslash
k}-\widetilde{\boldsymbol\beta}\|^2 +
\|\widetilde{\boldsymbol\alpha}^{\backslash
k}-\widetilde{\boldsymbol\alpha}\|^2\le\notag\\
&\frac4{\min(\lambda_1,\lambda_2)n}\left(\sum_{i=1}^m\|\widetilde{\beta}_i-\widetilde{\beta}^{\backslash
k}_i\| +
\sum_{i<j}|\widetilde{\alpha}_{ij}-\widetilde{\alpha}^{\backslash
k}_{ij}|\right).
\end{align}
Further, by using
\begin{align}
\|\widetilde{\boldsymbol\beta}^{\backslash
k}-\widetilde{\boldsymbol\beta}\|^2 &+
\|\widetilde{\boldsymbol\alpha}^{\backslash
k}-\widetilde{\boldsymbol\alpha}\|^2\ge\notag\\
&\frac12\left(\sum_{i=1}^m\|\widetilde{\beta}_i-\widetilde{\beta}^{\backslash
k}_i\| +
\sum_{i<j}|\widetilde{\alpha}_{ij}-\widetilde{\alpha}^{\backslash
k}_{ij}|\right)^2
\end{align}
we have
\begin{align}
\sum_{i=1}^m\|\widetilde{\beta}_i-\widetilde{\beta}^{\backslash k}_i\|
+ \sum_{i<j}|\widetilde{\alpha}_{ij}-\widetilde{\alpha}^{\backslash
k}_{ij}|\le\frac8{\min(\lambda_1,\lambda_2)n}
\end{align}
Since $\mathcal D^k$
and $\mathcal D^{\backslash k}$ differ from each other with only the $k$-th training
example, the same argument gives
\begin{align}
\sum_{i=1}^m\|\widetilde{\beta}^k_i-\widetilde{\beta}^{\backslash
k}_i\| +
\sum_{i<j}|\widetilde{\alpha}^k_{ij}-\widetilde{\alpha}^{\backslash
k}_{ij}|\le\frac8{\min(\lambda_1,\lambda_2)n}.
\end{align}
Then, (\ref{eq:estimation_stability}) is obtained immediately. This
completes the proof.
\end{proof}

\subsection{Generalization Bound}
We first define a loss function to measure the generalization error.
Considering that CorrLog predicts labels by MAP estimation, we
define the loss function by using the log probability
\begin{equation}\label{eq:loss}
\ell(\mathbf x,\mathbf y;{\Theta}) =
\left\{\begin{array}{ll}1,&f(\mathbf x,\mathbf y,\Theta)<0\\
1-f(\mathbf x,\mathbf y,\Theta)/\gamma,&0\le f(\mathbf x,\mathbf y,\Theta)<\gamma\\
0,&f(\mathbf x,\mathbf y,\Theta)\ge \gamma,
\end{array}\right.
\end{equation}
where the constant $\gamma>0$ and
\begin{align}\label{eq:f}
f(\mathbf x,\mathbf y,\Theta)&=\log p(\mathbf y|\mathbf
x;{\Theta})-\max\limits_{\mathbf y'\ne\mathbf y}\log
p(\mathbf y'|\mathbf x;{\Theta})\notag\\
& = \left(\sum_{i=1}^m\mathbf y_i\beta_i^T\mathbf x
+\sum_{i<j}\alpha_{ij}\mathbf y_i\mathbf y_j\right)\notag\\
&-\max_{\mathbf
y'\ne\mathbf y}\left(\sum_{i=1}^m\mathbf y_i'\beta_i^T\mathbf x
+\sum_{i<j}\alpha_{ij}\mathbf y_i'\mathbf y_j'\right).
\end{align}
The loss function (\ref{eq:loss}) is defined analogously to the loss
function used in binary classification, where $f(\mathbf x,\mathbf
y,\Theta)$ is replaced with the margin $y\mathbf w^T\mathbf x$ if a
linear classifier $\mathbf w$ is used. Besides, (\ref{eq:loss})
gives a 0 loss only if all dimensions of $\mathbf y$ are correctly
predicted, which emphasizes the joint prediction in MLC. By using
this loss function, the generalization error and the empirical error
are given by
\begin{equation}
\mathcal R (\widetilde{\Theta})= \mathbb E_{\mathbf x\mathbf
y}\ell(\mathbf x,\mathbf y;\widetilde{\Theta}),\end{equation} and
\begin{equation}
\widetilde{\mathcal R} (\widetilde{\Theta})= \frac1n\sum_{l=1}^n
\ell(\mathbf x^{(l)},\mathbf
y^{(l)};\widetilde{\Theta}).\end{equation}

According to \cite{StabilityAndGeneralization}, an exponential bound
exists for $\mathcal R(\widetilde{\Theta})$ if CorrLog has a uniform
stability with respect to the loss function (\ref{eq:loss}). The
following Theorem 2 shows this condition holds.
\begin{mytheorem}\label{lemma:uniform_stability}
Given model parameters $\widetilde\Theta=\{\widetilde{\boldsymbol\beta},\widetilde{\boldsymbol\alpha}\}$ and $\widetilde\Theta^{k}=\{\widetilde{\boldsymbol\beta}^{k},\widetilde{\boldsymbol\alpha}^{k}\}$ learned on training datasets $ \mathcal D$ and $\mathcal D^k$, respectively, both by (\ref{eq:estimation}), it holds for $\forall (\mathbf x,\mathbf y)\in\mathcal X\times\mathcal Y$,
\begin{equation}\label{eq:uniformstability}
|\ell(\mathbf x,\mathbf y;\widetilde{\Theta})-\ell(\mathbf x,\mathbf
y;\widetilde{\Theta}^{k})|
\le\frac{32}{\gamma\min(\lambda_1,\lambda_2)n}.
\end{equation}
\end{mytheorem}

\begin{proof}
First, we have the following inequality from (\ref{eq:loss})
\begin{align}
\gamma|\ell(\mathbf
x,\mathbf y;\widetilde{\Theta})-\ell(\mathbf x,\mathbf
y;\widetilde{\Theta}^{k})|\le |f(\mathbf x,\mathbf
y,\widetilde{\Theta})-f(\mathbf x,\mathbf y,\widetilde{\Theta}^{k})|
\end{align}
Then, by introducing notation
\begin{align}
A(\mathbf x,\mathbf
y,{\boldsymbol\beta},{\boldsymbol\alpha})=\sum_{i=1}^m\mathbf
y_i\beta_i^T\mathbf x +\sum_{i<j}\alpha_{ij}\mathbf y_i\mathbf y_j,
\end{align}
and rewriting
\begin{align}
f(\mathbf x,\mathbf y,{\Theta})=A(\mathbf x,\mathbf
y,{\boldsymbol\beta},\boldsymbol\alpha)-\max_{\mathbf y'\ne\mathbf
y}A(\mathbf x,\mathbf y',{\boldsymbol\beta},{\boldsymbol\alpha}),
\end{align}
we have
\begingroup\makeatletter\def\f@size{9}\check@mathfonts 
\begin{align}
&\gamma|\ell(\mathbf x,\mathbf y;\widetilde{\Theta})-\ell(\mathbf
x,\mathbf y;\widetilde{\Theta}^{k})|\le\big|A(\mathbf x,\mathbf
y,\widetilde{\boldsymbol\beta},\widetilde{\boldsymbol\alpha})-A(\mathbf
x,\mathbf
y,\widetilde{\boldsymbol\beta}^{k},\widetilde{\boldsymbol\alpha}^{k})\big|\notag\\
&\qquad\quad+ |\max\limits_{\mathbf y'\ne\mathbf y}A(\mathbf x,\mathbf
y',\widetilde{\boldsymbol\beta},\widetilde{\boldsymbol\alpha})-\max\limits_{\mathbf
y'\ne\mathbf y}A(\mathbf x,\mathbf
y',\widetilde{\boldsymbol\beta}^{k},\widetilde{\boldsymbol\alpha}^{k})|.
\end{align}
\endgroup 
Due to the fact that for any functions $h_1(u)$ and $h_2(u)$ it holds\footnote{
Suppose $u_1^{\star}$ and $u_2^{\star}$ maximize $h_1(u)$ and
$h_2(u)$ respectively, and without loss of generality
$h_1(u_1^{\star})\ge h_2(u_2^{\star})$, we have
$|h_1(u_1^{\star})-h_2(u_2^{\star})|=h_1(u_1^{\star})-h_2(u_2^{\star})\le
h_1(u_1^{\star})-h_2(u_1^{\star})\le\max_{u}|h_1(u)-h_2(u)|$.}
\begin{align}
|\max_{u}h_1(u)-\max_{u}h_2(u)|\le\max_{u}|h_1(u)-h_2(u)|,
\end{align}
we
have
\begin{align}
&\gamma|\ell(\mathbf x,\mathbf y;\widetilde{\Theta})-\ell(\mathbf
x,\mathbf y;\widetilde{\Theta}^{k})|\notag\\
&\le\big|A(\mathbf x,\mathbf
y,\widetilde{\boldsymbol\beta},\widetilde{\boldsymbol\alpha})-A(\mathbf
x,\mathbf
y,\widetilde{\boldsymbol\beta}^{k},\widetilde{\boldsymbol\alpha}^{k})\big|\notag\\
&\qquad\qquad+ \max_{\mathbf y'\ne\mathbf y}\big|A(\mathbf x,\mathbf
y',\widetilde{\boldsymbol\beta},\widetilde{\boldsymbol\alpha})-A(\mathbf
x,\mathbf y',\widetilde{\boldsymbol\beta}^{
k},\widetilde{\boldsymbol\alpha}^{k})\big|\notag \\
&\le 2\max_{\mathbf y}\left(\sum_{i=1}^m|\mathbf
y_i(\widetilde{\beta}_i-\widetilde{\beta}_i^{k})^T\mathbf
x|+\sum_{i<j}|(\widetilde{\alpha}_{ij}-\widetilde{\alpha}_{ij}^{
k})\mathbf y_i\mathbf
y_j|\right)\notag\\
&\le2\left(\sum_{i=1}^m\|\widetilde{\beta}_i-\widetilde{\beta}_i^{
k}\|+2\sum_{i<j}|\widetilde{\alpha}_{ij}-\widetilde{\alpha}_{ij}^{k}|\right).
\end{align}
Then, the proof is completed by applying Theorem 1.
\end{proof}

Now, we are ready to present the main theorem on the generalization ability of CorrLog.
\begin{mytheorem}
Given the model parameter $\widetilde\Theta$ learned by (\ref{eq:estimation}), with i.i.d. training data $\mathcal D=\{(\mathbf x^{(l)},\mathbf
y^{(l)})\in\mathcal X\times \mathcal Y,l=1,2,...,n\}$ and
regularization parameters $\lambda_1$, $\lambda_2$, it holds with at
least probability $1-\delta$,
\begin{align}\label{eq:bound}
\mathcal R(\widetilde{\Theta})\le \widetilde{\mathcal
R}(&\widetilde{\Theta})+\frac{32}{\gamma\min(\lambda_1,\lambda_2)n}\notag\\
&+\left(\frac{64}{\gamma\min(\lambda_1,\lambda_2)}+1\right)\sqrt{\frac{\log1/\delta}{2n}}.
\end{align}
\end{mytheorem}
\begin{proof}
Given Theorem 2, the generalization bound (\ref{eq:bound}) is a direct result of Theorem 12 in
\cite{StabilityAndGeneralization} (Please refer to the reference for details).
\end{proof}
\noindent
\textbf{Remark 3} A notable observation from Theorem 3 is that the
generalization bound (\ref{eq:bound}) of CorrLog is independent of the label number
$m$. Therefore, CorrLog is preferable for MLC with a large number of labels, for which the generalization error still can be bounded with high
confidence.

\noindent
\textbf{Remark 4} While the learning of CorrLog (\ref{eq:estimation}) utilizes the elastic net regularization $R_{en}(\Theta;\lambda_1,\lambda_2,\epsilon)$, where $\epsilon$ is the weighting parameter on the $\ell_1$ regularization to encourage sparsity, the generalization bound (\ref{eq:bound}) is independent of the parameter $\epsilon$. The reason is that $\ell_1$ regularization does not lead to stable learning algorithms \cite{XuCM12}, and only the $\ell_2$ regularization in $R_{en}(\Theta;\lambda_1,\lambda_2,\epsilon)$ contributes to the stability of CorrLog.

\section{Toy Example}
\begin{figure*}[t]
\centering
\subfigure[]{\includegraphics[width=0.66\columnwidth]{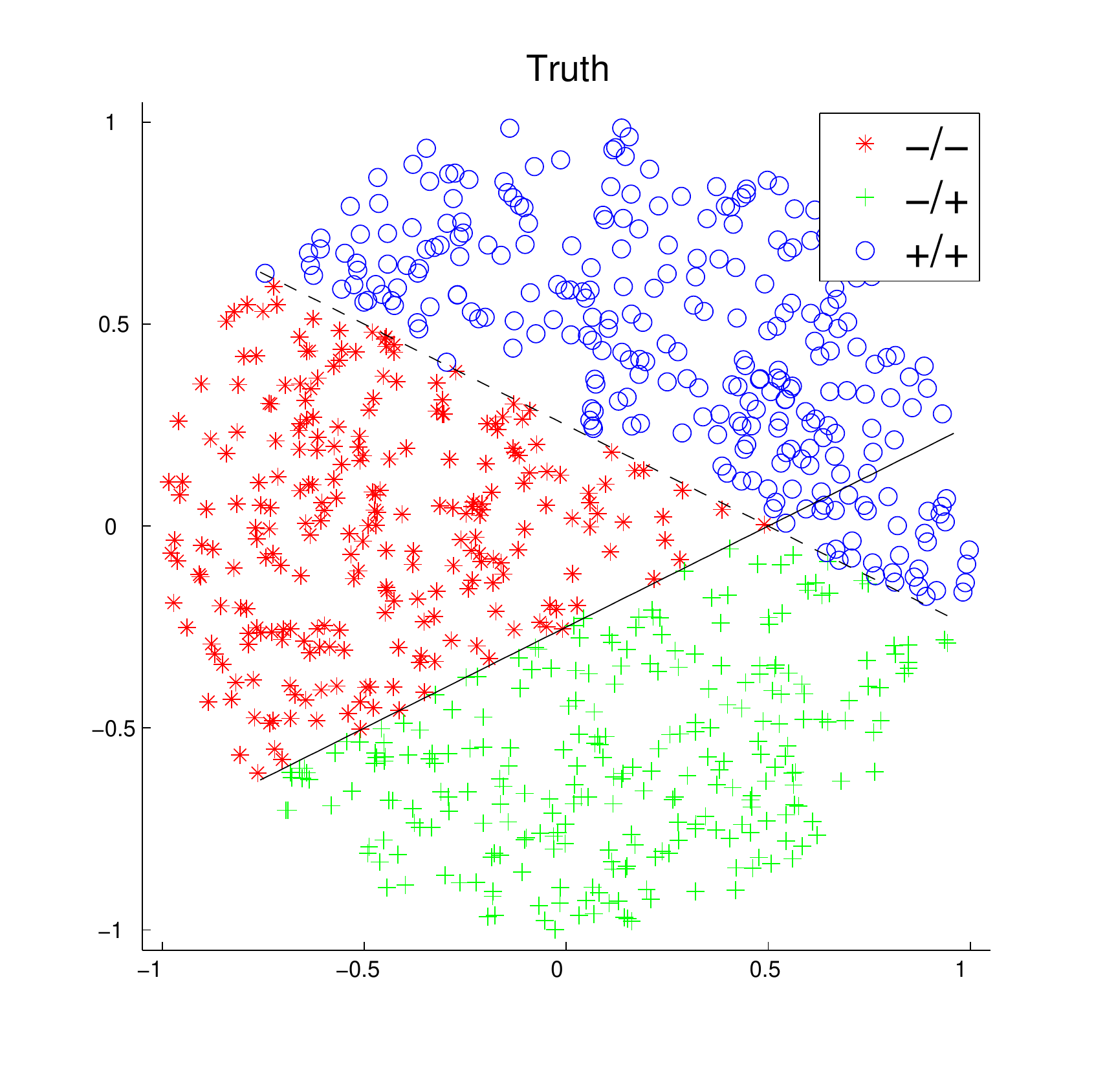}}
\subfigure[]{\includegraphics[width=0.66\columnwidth]{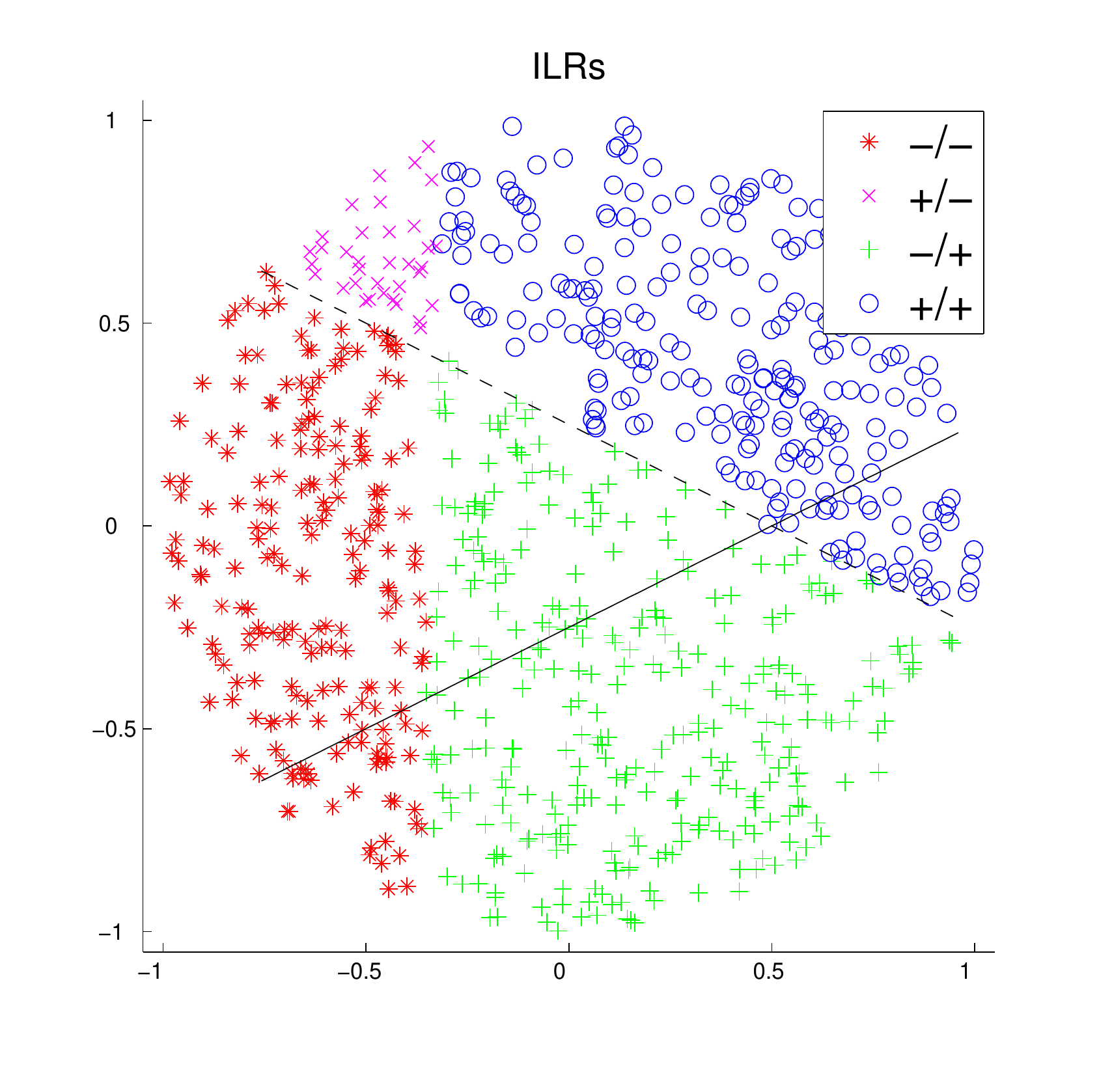}}
\subfigure[]{\includegraphics[width=0.66\columnwidth]{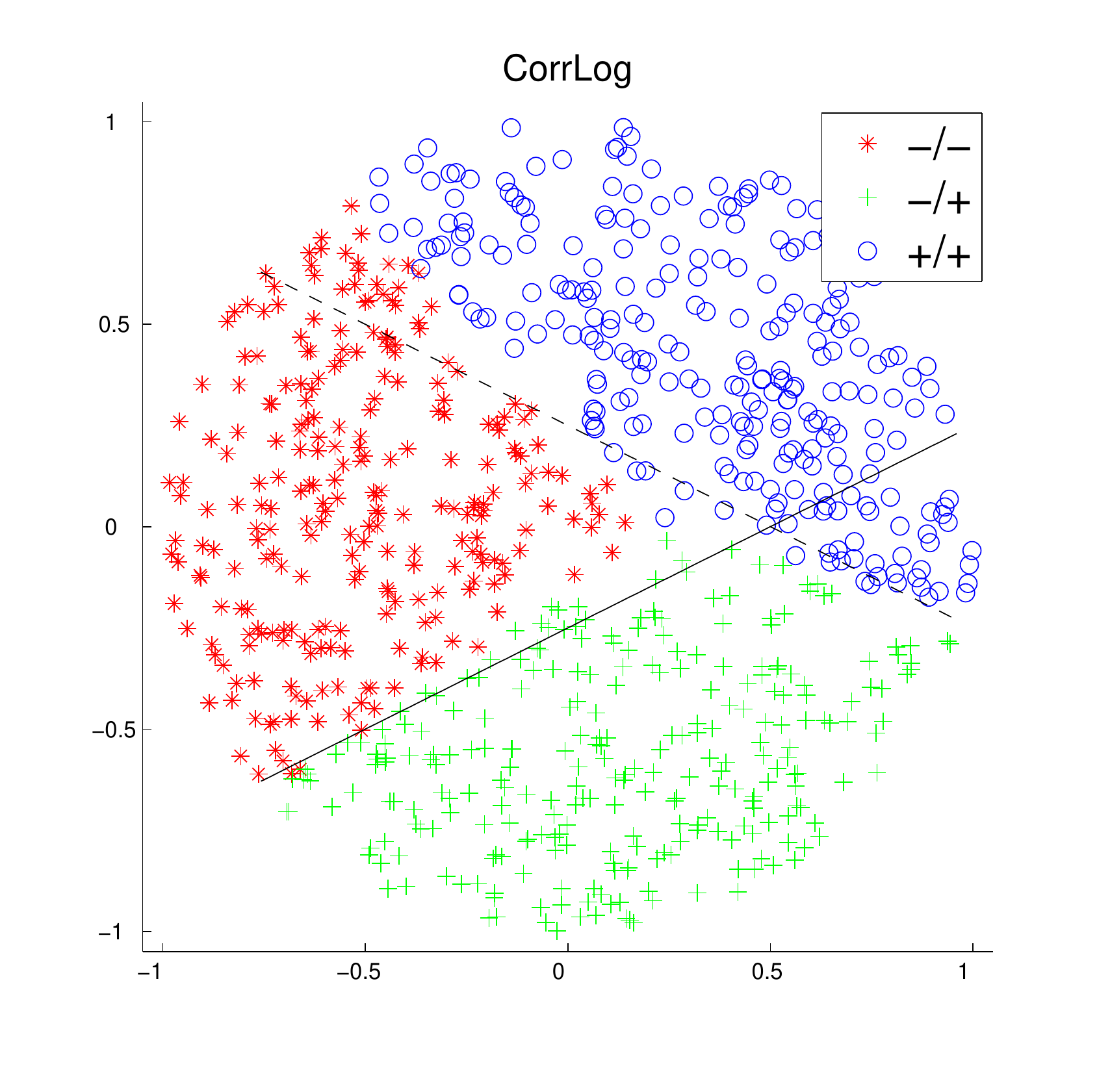}}
\caption{A two-label toy example: (a) true labels of test data; (b) predictions given by ILRs; (c) predictions given by CorrLog.
The dash and solid black boundaries are specified by $\eta_1$ and $\eta_2$.
In the legend, ``$+$'' and ``-'' stand for positive and negative labels, respectively,
e.g., ``$-/+$'' means $\mathbf y_1=-1$ and $\mathbf y_2=1$, and so on.}
\label{fig:toy}
\end{figure*}

We design a simple toy example to illustrate the capacity of CorrLog on label correlation modelling.
In particular, we show that when ILRs fail drastically due to ignoring the label correlations (under-fitting), CorrLog performs well.
Consider a two-label classification problem on a 2-D plane,
where each instance $\mathbf x$ is sampled uniformly from the unit disc $\|\mathbf x\|\le1$
and the corresponding labels $\mathbf y=[\mathbf y_1,\mathbf y_2]$ are defined by
\begin{equation}\notag
\mathbf y_1 = \mathrm{sign}(\eta_1^T\tilde{\mathbf x}) \mbox{~ and ~}
\mathbf y_2 = \mathrm{OR}\left(\mathbf y_1,\mathrm{sign}(\eta_2^T\tilde{\mathbf x})\right),
\end{equation}
where $\eta_1=(1,1, -0.5)$, $\eta_2=(-1,1, -0.5)$
and the augmented feature is $\tilde{\mathbf x}=[\mathbf x^T, 1]^T$.
The $\mathrm{sign}(\cdot)$ function takes value $1$ or $-1$,
and the $\rm OR(\cdot,\cdot)$ operation outputs $1$ if either of its input is $1$.
The definition of $\mathbf y_2$ makes the two labels correlated.
We generate 1000 random examples according
to above setting and split them into training and test sets, each of which contains 500 examples.
During training, we set the parameter $\epsilon$ of the elastic net regularization to 0,
i.e., we actually used an $\ell_2$ regularization,
this is because in this example the model is not sparse in terms of both feature selection and label correlation.
In addition, as the number of the training examples is sufficiently large for this problem,
we suppose there is no over-fitting and tune the regularization parameters for both ILRs and CorrLog
by minimizing the 0-1 loss on the training set.

Figure \ref{fig:toy} shows that true labels of test data, the
predictions of ILRs and the predictions of CorrLog, where different
labels are marked by different colors. In (a), the disc is divided
into three regions, $-/-$, $-/+$ and $+/+$, where the two black
boundaries are specified by $\eta_1$ and $\eta_2$, respectively. In
(b), the first boundary $\eta_1$ is properly learned by ILRs, while the
second one is learned wrongly. This is because the second label is
highly correlated to the first label, but ILRs ignores such
correlation.
As a result, ILRs wrongly predicted the impossible case of $+/-$.
The misclassification rate measured by 0-1 loss is
0.197. In contrast, CorrLog predicts correct labels for most
instances with a 0-1 loss 0.068. Besides, it is interesting to note
that the correlation between the two labels are ``asymmetric'', for
the first label is not affected by the second. This asymmetry
contributes the most to the misclassification of CorrLog, because
the previous definition implies that only symmetric correlations are
modelled in CorrLog.



\section{Multilabel Image Classification}
In this section, we apply the proposed CorrLog to multilabel image classification. In particular, four multilabel image datasets are used in this paper, including
MULAN scene (MULANscene)\footnote{http://mulan.sourceforge.net/},
MIT outdoor scene (MITscene) \cite{oliva2001modeling}, PASCAL VOC 2007 (PASCAL07)
\cite{everingham2010pascal} and PASCAL VOC 2012 (PASCAL12) \cite{everingham2015pascal}.
MULAN scene dataset contains 2047 images with 6 labels, and each image is represented by 294 features.
MIT outdoor scene dataset contains 2688 images in 8 categories. To make it suitable for multilabel experiment, we transformed each category label with several tags according to the image contents of
that category\footnote{The 8 categories are coast, forest, highway, insidecity, mountain, opencountry, street, and tallbuildings. The 8 binary tags are building, grass, cement-road, dirt-road, mountain, sea, sky, and tree.
The transformation follows, $C1\rightarrow(B6,B7)$, $C2\rightarrow(B4,B8)$, $C3\rightarrow(B3,B7)$, $C4\rightarrow(B1)$, $C5\rightarrow(B5,B7)$, $C6\rightarrow(B2,B4,B7)$, $C7\rightarrow(B1,B3,B7)$, $C8\rightarrow(B1, B7)$.
For example, coast $(C1)$ is tagged with sea $(B6)$ and sky $(B7)$.}.
PASCAL VOC 2007 dataset consists of 9963 images with 20 labels.
For PASCAL VOC 2012, we use the available train-validation subset which contains 11540 images.
In addition, two kinds of features are adopted to represent the last three datasets, i.e., the PHOW (a variant of dense SIFT descriptors extracted at multiple scales) features \cite{bosch2007image} and deep CNN (convolutional neural network) features \cite{krizhevsky2012imagenet,Chatfield14}.
Summary of the basic information of the datasets is illustrated in Table \ref{TabData}.
To extract PHOW features, we use the VLFeat implementation \cite{vedaldi2010vlfeat}. For deep CNN features, we use the 'imagenet-vgg-f' model pretrained on ImageNet database \cite{Chatfield14} which is available in MatConvNet matlab toolbox \cite{vedaldi2015matconvnet}.

\begin{table}
\caption{Datasets summary. \#images stands for the number of all images, \#features stands for the dimension of the features, and \#labels stands for the number of labels.}
\label{TabData}
\centering
\begin{tabular}{|c|c|c|c|}\hline
Datasets & \#images & \#features & \#labels \\\hline
MULANscene & 2047 & 294 & 6 \\
MITscene-PHOW & 2688 & 3600 & 8 \\
MITscene-CNN & 2688 & 4096 & 8 \\
PASCAL07-PHOW & 9963 & 3600 & 20 \\
PASCAL07-CNN & 9963 & 4096 & 20 \\
PASCAL12-PHOW & 11540 & 3600 & 20 \\
PASCAL12-CNN & 11540 & 4096 & 20 \\
\hline
\end{tabular}
\end{table}

\subsection{A Warming-Up Qualitative Experiment}

\begin{table*}
\caption{Learned CorrLog label graph on MITscene using $\ell_2$ or elastic net regularization.}
\label{TabCorrLog}
\begin{center}
\renewcommand\tabcolsep{2pt} 
\renewcommand\arraystretch{1.2} 
\begin{tabular}{ |c|c|c|c|c|c|c|c| }
\hline
\mc{8}{|c|}{\textbf{MITscene Images and Tags}}\\\hline
coast & forest & highway & insidecity & mountain & opencountry & street & tallbuilding\\
\includegraphics[width=0.115\textwidth]{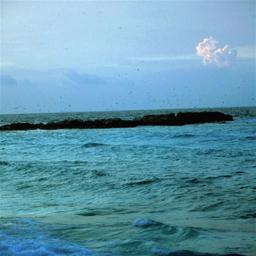}
& \includegraphics[width=0.115\textwidth]{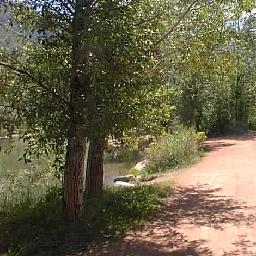}
& \includegraphics[width=0.115\textwidth]{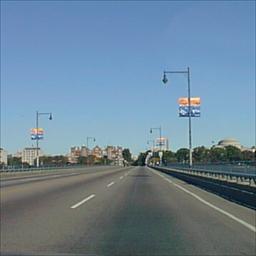}
& \includegraphics[width=0.115\textwidth]{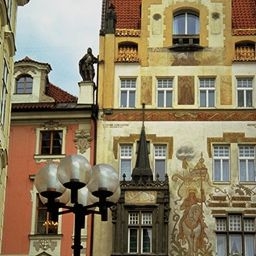}
& \includegraphics[width=0.115\textwidth]{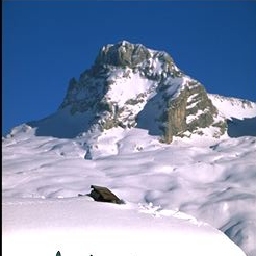}
& \includegraphics[width=0.115\textwidth]{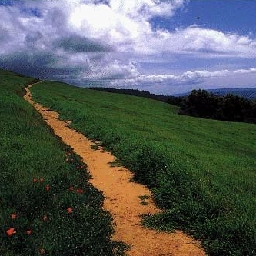}
& \includegraphics[width=0.115\textwidth]{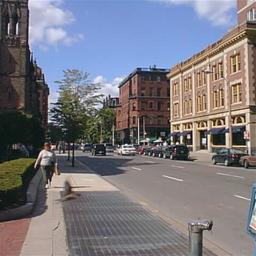}
& \includegraphics[width=0.115\textwidth]{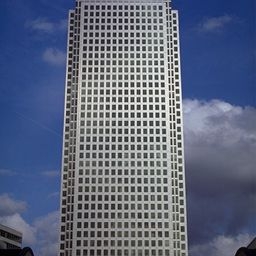}\\\hline
\begin{tabular}{@{}c@{}}sea \\ sky\end{tabular}
& \begin{tabular}{@{}c@{}}dirt-road \\ tree\end{tabular}
& \begin{tabular}{@{}c@{}}cement-road \\ sky\end{tabular}
& \begin{tabular}{@{}c@{}}building\end{tabular}
& \begin{tabular}{@{}c@{}}mountain \\ sky\end{tabular}
& \begin{tabular}{@{}c@{}}grass \\ dirt-road \\ sky\end{tabular}
& \begin{tabular}{@{}c@{}}building \\ cement-road \\ sky\end{tabular}
& \begin{tabular}{@{}c@{}}building \\ sky\end{tabular}\\
\hline\hline
\mc{8}{|c|}{\textbf{Learned CorrLog Label Graph}}\\\hline
\mc{4}{|c}{$\ell_2$ regularization} & \mc{4}{c|}{Elastic net regularization}\\
\mc{8}{|c|}{\includegraphics[width=0.9\textwidth]{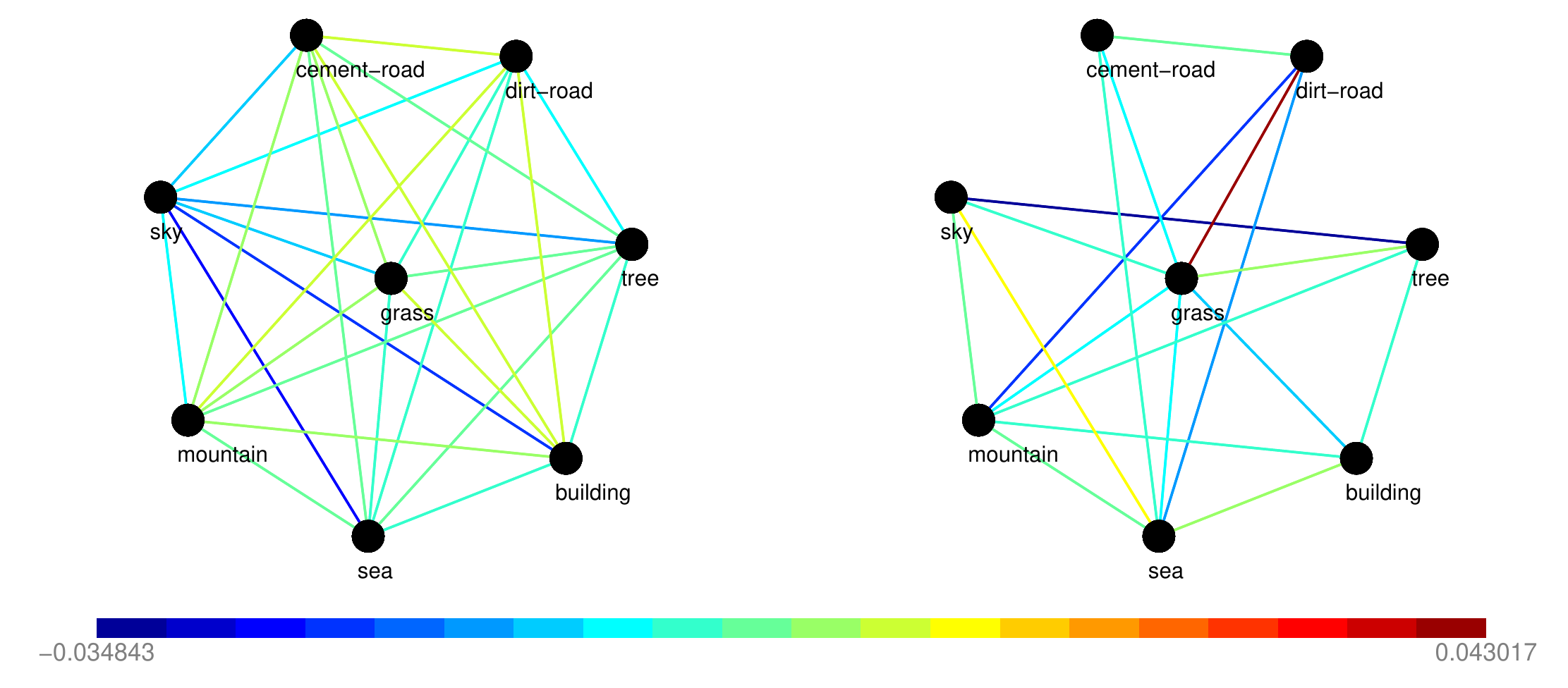}}\\
\hline
\end{tabular}
\end{center}
\end{table*}

As an extension to our previous work on CorrLog, this paper utilizes elastic net to inherit individual advantages of $\ell_2$ and $\ell_1$ regularization. To build up the intuition, we employ MITscene with PHOW features to visualize the difference between $\ell_2$ and elastic net regularization. Table \ref{TabCorrLog} presents the learned CorrLog label graphs using these two types of regularization respectively. In the label graph, the color of each edge represents the correlation strength between two certain labels. We have also listed $8$ representative example images, one for each category, and their binary tags for completeness.

According to the comparison, one can see that elastic net regularization results in a sparse label graph due to its $\ell_1$ component, while $\ell_2$ regularization can only lead to a fully-connected label graph.
In addition, the learned label correlations in elastic net case are more reasonable than that of $\ell_2$.
For example, in the $\ell_2$ label graph, dirt-road and mountain have weekly positive correlation (according to the link between them), though they seldom co-occur on the images in the datasets, while in the elastic net graph, their correlation is corrected as negative.
It has to be confessed that elastic net regularization also discarded some reasonable correlations such as cement-road and building. This phenomenon is a direct result of the compromise between learning stability and model sparsity.
We shall mention that those reasonable correlations can be maintained by decreasing $\lambda_1$, $\lambda_2$ or $\epsilon$, though more unreasonable connections will also be maintained. Thus, applying weak sparsity may impair the model performance.
As a result, it is important to choose a good level of sparsity to achieve a compromise.
In our experiments, CorrLog with elastic net regularization generally outperforms that with $\ell_2$ regularization, which confirms our motivation that appropriate level of sparsity in feature selection and label correlations help boost the performance of MLC.
In the following presentation, we will use CorrLog with elastic net regularization in all experimental comparisons.
To benefit following research, our code is available upon request.

\subsection{Quantitative Experimental Setting}
In this subsection, we present further comparisons between CorrLog and other MLC methods.
First, to demonstrate the effectiveness of utilizing label correlation, we first compare CorrLog's performance with ILRs.
Moreover, four state-of-the-art MLC methods - instance-based learning by logistic regression (IBLR) \cite{IBLR}, multilabel k-nearest neighbour (MLkNN) \cite{MLKNN}, classifier chains (CC) \cite{ClassifierChain}
and maximum margin output coding (MMOC) \cite{zhang2012maximum}
were also employed for comparison study.
Note that ILRs can be regarded as the basic baseline and other methods represent state-of-the-arts.
In our experiments, LIBlinear \cite{fan2008liblinear} $\ell_2$-regularized logistic regression is employed to build binary classifiers for ILRs.
As for other methods, we use publicly available codes in MEKA\footnote{http://meka.sourceforge.net/}
or the authors' homepages.

We used six different measures to evaluate the performance.
These include different loss functions (Hamming loss and zero-one loss) and other popular measures (accuracy, F1 score, Macro-F1 and Micro-F1).
The details of these evaluation measures can be found in \cite{madjarov2012extensive, Rakel, ClassifierChain, PCC}.
The parameters for CorrLog are fixed across all experiments as $\lambda_1=0.001$, $\lambda_2=0.001$ and $\epsilon=1$.
On each dataset, all the methods are compared by 5-fold cross validation.
The mean and standard deviation are reported for each criterion.
In addition, paired t-tests at 0.05 significance level is applied to evaluate the statistical significance of performance difference.

\subsection{Quantitative Results and Discussions}
\begin{table*}
\caption{MULANscene performance comparison via 5-fold cross validation.
Marker $\ast/\circledast$ indicates whether CorrLog is statistically superior/inferior to the compared method (using paired t-test at 0.05 significance level).}
\label{TabMULANscene}
\begin{center}
\begin{tabular}{|c|c|l|l|l|l|l|l|}\hline
\mr{2}{*}{Datasets} & \mr{2}{*}{Methods} & \mc{6}{c|}{Measures} \\\cline{3-8}
 &  & \mc{1}{c|}{Hamming loss} & \mc{1}{c|}{0-1 loss} & \mc{1}{c|}{Accuracy} & \mc{1}{c|}{F1-Score} & \mc{1}{c|}{Macro-F1} & \mc{1}{c|}{Micro-F1} \\\hline
\mr{8}{*}{MULANscene} & CorrLog & 0.095$\pm$0.007 & \textbf{0.341$\pm$0.020} & \textbf{0.710$\pm$0.018} & \textbf{0.728$\pm$0.017} & 0.745$\pm$0.016 & 0.734$\pm$0.017 \\
 & ILRs & 0.117$\pm$0.006 $\ast$ & 0.495$\pm$0.022 $\ast$ & 0.592$\pm$0.016 $\ast$ & 0.622$\pm$0.014 $\ast$ & 0.677$\pm$0.016 $\ast$ & 0.669$\pm$0.014 $\ast$ \\
 & IBLR & \textbf{0.085$\pm$0.004} $\circledast$ & 0.358$\pm$0.016 & 0.677$\pm$0.018 $\ast$ & 0.689$\pm$0.019 $\ast$ & \textbf{0.747$\pm$0.010} & \textbf{0.738$\pm$0.014} \\
 & MLkNN & 0.086$\pm$0.003 & 0.374$\pm$0.015 $\ast$ & 0.668$\pm$0.018 $\ast$ & 0.682$\pm$0.019 $\ast$ & 0.742$\pm$0.013 & 0.734$\pm$0.012 \\
 & CC & 0.104$\pm$0.005 $\ast$ & 0.346$\pm$0.015 & 0.696$\pm$0.015 $\ast$ & 0.710$\pm$0.015 $\ast$ & 0.716$\pm$0.018 $\ast$ & 0.706$\pm$0.014 $\ast$ \\
 & MMOC & 0.126$\pm$0.017 $\ast$ & 0.401$\pm$0.046 $\ast$ & 0.629$\pm$0.049 $\ast$ & 0.639$\pm$0.050 $\ast$ & 0.680$\pm$0.031 $\ast$ & 0.638$\pm$0.049 $\ast$ \\
\hline
\end{tabular}
\end{center}
\end{table*}

\begin{table*}
\caption{MITscene performance comparison via 5-fold cross validation.
Marker $\ast/\circledast$ indicates whether CorrLog is statistically superior/inferior to the compared method (using paired t-test at 0.05 significance level).}
\label{TabMITscene}
\begin{center}
\begin{tabular}{|c|c|l|l|l|l|l|l|}\hline
\mr{2}{*}{Datasets} & \mr{2}{*}{Methods} & \mc{6}{c|}{Measures} \\\cline{3-8}
 &  & \mc{1}{c|}{Hamming loss} & \mc{1}{c|}{0-1 loss} & \mc{1}{c|}{Accuracy} & \mc{1}{c|}{F1-Score} & \mc{1}{c|}{Macro-F1} & \mc{1}{c|}{Micro-F1} \\\hline
\mr{8}{*}{MITscene-PHOW} & CorrLog & 0.045$\pm$0.006 & \textbf{0.196$\pm$0.017} & \textbf{0.884$\pm$0.012} & \textbf{0.914$\pm$0.010} & 0.883$\pm$0.017 & \textbf{0.915$\pm$0.011} \\
 & ILRs & 0.071$\pm$0.002 $\ast$ & 0.358$\pm$0.015 $\ast$ & 0.825$\pm$0.007 $\ast$ & 0.877$\pm$0.005 $\ast$ & 0.833$\pm$0.007 $\ast$ & 0.872$\pm$0.003 $\ast$ \\
 & IBLR & 0.060$\pm$0.003 $\ast$ & 0.243$\pm$0.021 $\ast$ & 0.845$\pm$0.012 $\ast$ & 0.879$\pm$0.008 $\ast$ & 0.848$\pm$0.009 $\ast$ & 0.886$\pm$0.006 $\ast$ \\
 & MLkNN & 0.069$\pm$0.002 $\ast$ & 0.326$\pm$0.022 $\ast$ & 0.810$\pm$0.009 $\ast$ & 0.857$\pm$0.006 $\ast$ & 0.827$\pm$0.009 $\ast$ & 0.869$\pm$0.004 $\ast$ \\
 & CC & 0.047$\pm$0.005 & 0.200$\pm$0.021 & 0.883$\pm$0.012 & 0.913$\pm$0.008 & \textbf{0.883$\pm$0.015} & 0.913$\pm$0.009 \\
 & MMOC & 0.062$\pm$0.010 $\ast$ & 0.274$\pm$0.035 $\ast$ & 0.845$\pm$0.017 $\ast$ & 0.885$\pm$0.014 $\ast$ & 0.846$\pm$0.024 $\ast$ & 0.885$\pm$0.017 $\ast$ \\
\hline
\mr{8}{*}{MITscene-CNN} & CorrLog & \textbf{0.017$\pm$0.004} & 0.088$\pm$0.015 & 0.953$\pm$0.008 & 0.966$\pm$0.006 & \textbf{0.957$\pm$0.011} & \textbf{0.968$\pm$0.006} \\
 & ILRs & 0.020$\pm$0.002 $\ast$ & 0.102$\pm$0.015 $\ast$ & 0.947$\pm$0.006 $\ast$ & 0.962$\pm$0.004 $\ast$ & 0.951$\pm$0.007 $\ast$ & 0.963$\pm$0.005 $\ast$ \\
 & IBLR & 0.022$\pm$0.001 $\ast$ & 0.090$\pm$0.009 & 0.944$\pm$0.004 & 0.957$\pm$0.003 $\ast$ & 0.944$\pm$0.004 $\ast$ & 0.958$\pm$0.003 $\ast$ \\
 & MLkNN & 0.024$\pm$0.002 $\ast$ & 0.104$\pm$0.005 $\ast$ & 0.939$\pm$0.003 $\ast$ & 0.954$\pm$0.003 $\ast$ & 0.941$\pm$0.002 $\ast$ & 0.955$\pm$0.004 $\ast$ \\
 & CC & 0.021$\pm$0.003 $\ast$ & 0.075$\pm$0.008 $\circledast$ & 0.951$\pm$0.005 & 0.962$\pm$0.004 $\ast$ & 0.948$\pm$0.007 $\ast$ & 0.961$\pm$0.005 $\ast$ \\
 & MMOC & 0.018$\pm$0.002 & \textbf{0.062$\pm$0.005} $\circledast$ & \textbf{0.959$\pm$0.003} $\circledast$ & \textbf{0.967$\pm$0.003} & 0.955$\pm$0.005 & 0.967$\pm$0.004 \\
\hline
\end{tabular}
\end{center}
\end{table*}

Tables \ref{TabMULANscene}, \ref{TabMITscene}, \ref{TabPAS07} and \ref{TabPAS12} summarized the experimental results on MULANscene, MITscene, PASCAL07 and PASCAL12 of all six algorithms evaluated by the six measures.
By comparing the results of CorrLog and ILRs,
we can clearly see the improvements obtained by exploiting label correlations for MLC.
Except the Hamming loss, CorrLog greatly outperforms ILRs on all datasets.
Especially, the reduction of zero-one loss is significant on all four datasets with different type of features.
This confirms the value of correlation modelling to joint prediction.
However, it should be noticed that the improvement of CorrLog over ILRs is less significant when the performance is measured by Hamming loss.
This is because Hamming loss treats the prediction of each label individually.

In addition, CorrLog is more effective in exploiting label correlations than other four state-of-the-art MLC algorithms.
For MULANscene dataset, CorrLog achieved comparable results with IBLR and both of them outperformed other methods.
For MITscene dataset, both PHOW and CNN features are very effective representations and boost the classification results.
As a consequence, the performance of CorrLog and the four MLC algorithms are very close to each other.
It is worth noting that, the MMOC method is time-consuming in the training stage, though it achieved the best performance on this dataset.
As for both PASCAL07 and PASCAL12 datasets, CNN features perform significantly better than PHOW features.
CorrLog obtained much better results than the competing MLC schemes, except for the Hamming loss and zero-one loss.
Note that the CorrLog also performs competitively with PLEM and CGM, according to the results reported in \cite{tan2015learning}.

\begin{table*}
\caption{PASCAL07 performance comparison via 5-fold cross validation.
Marker $\ast/\circledast$ indicates whether CorrLog is statistically superior/inferior to the compared method (using paired t-test at 0.05 significance level).}
\label{TabPAS07}
\begin{center}
\begin{tabular}{|c|c|l|l|l|l|l|l|}\hline
\mr{2}{*}{Datasets} & \mr{2}{*}{Methods} & \mc{6}{c|}{Measures} \\\cline{3-8}
 &  & \mc{1}{c|}{Hamming loss} & \mc{1}{c|}{0-1 loss} & \mc{1}{c|}{Accuracy} & \mc{1}{c|}{F1-Score} & \mc{1}{c|}{Macro-F1} & \mc{1}{c|}{Micro-F1} \\\hline
\mr{8}{*}{PASCAL07-PHOW} & CorrLog & 0.068$\pm$0.001 & \textbf{0.776$\pm$0.007} & \textbf{0.370$\pm$0.010} & \textbf{0.423$\pm$0.012} & \textbf{0.367$\pm$0.011} & \textbf{0.480$\pm$0.008} \\
 & ILRs & 0.093$\pm$0.001 $\ast$ & 0.878$\pm$0.007 $\ast$ & 0.294$\pm$0.008 $\ast$ & 0.360$\pm$0.009 $\ast$ & 0.332$\pm$0.008 $\ast$ & 0.404$\pm$0.007 $\ast$ \\
 & IBLR & 0.066$\pm$0.001 $\circledast$ & 0.832$\pm$0.003 $\ast$ & 0.270$\pm$0.005 $\ast$ & 0.308$\pm$0.006 $\ast$ & 0.258$\pm$0.007 $\ast$ & 0.408$\pm$0.009 $\ast$ \\
 & MLkNN & 0.066$\pm$0.001 $\circledast$ & 0.839$\pm$0.006 $\ast$ & 0.256$\pm$0.007 $\ast$ & 0.291$\pm$0.008 $\ast$ & 0.235$\pm$0.006 $\ast$ & 0.392$\pm$0.007 $\ast$ \\
 & CC & 0.091$\pm$0.000 $\ast$ & 0.845$\pm$0.010 $\ast$ & 0.318$\pm$0.005 $\ast$ & 0.379$\pm$0.003 $\ast$ & 0.348$\pm$0.004 $\ast$ & 0.417$\pm$0.001 $\ast$ \\
 & MMOC & \textbf{0.065$\pm$0.001} $\circledast$ & 0.850$\pm$0.003 $\ast$ & 0.259$\pm$0.009 $\ast$ & 0.299$\pm$0.011 $\ast$ & 0.206$\pm$0.007 $\ast$ & 0.392$\pm$0.012 $\ast$ \\
\hline
\mr{8}{*}{PASCAL07-CNN} & CorrLog & 0.038$\pm$0.001 & 0.516$\pm$0.010 & \textbf{0.642$\pm$0.010} & \textbf{0.696$\pm$0.010} & \textbf{0.674$\pm$0.002} & \textbf{0.724$\pm$0.006} \\
 & ILRs & 0.046$\pm$0.001 $\ast$ & 0.574$\pm$0.011 $\ast$ & 0.610$\pm$0.010 $\ast$ & 0.673$\pm$0.009 $\ast$ & 0.651$\pm$0.004 $\ast$ & 0.688$\pm$0.007 $\ast$ \\
 & IBLR & 0.043$\pm$0.001 $\ast$ & 0.554$\pm$0.011 $\ast$ & 0.597$\pm$0.014 $\ast$ & 0.649$\pm$0.015 $\ast$ & 0.621$\pm$0.007 $\ast$ & 0.682$\pm$0.010 $\ast$ \\
 & MLkNN & 0.043$\pm$0.001 $\ast$ & 0.557$\pm$0.010 $\ast$ & 0.585$\pm$0.014 $\ast$ & 0.635$\pm$0.015 $\ast$ & 0.613$\pm$0.006 $\ast$ & 0.668$\pm$0.011 $\ast$ \\
 & CC & 0.051$\pm$0.001 $\ast$ & 0.586$\pm$0.008 $\ast$ & 0.602$\pm$0.008 $\ast$ & 0.668$\pm$0.008 $\ast$ & 0.635$\pm$0.009 $\ast$ & 0.669$\pm$0.008 $\ast$ \\
 & MMOC & \textbf{0.037$\pm$0.000} $\circledast$ & \textbf{0.512$\pm$0.008} & 0.634$\pm$0.009 $\ast$ & 0.684$\pm$0.009 $\ast$ & 0.663$\pm$0.005 $\ast$ & 0.719$\pm$0.004 $\ast$ \\
\hline
\end{tabular}
\end{center}
\end{table*}

\begin{table*}
\caption{PASCAL12 performance comparison via 5-fold cross validation.
Marker $\ast/\circledast$ indicates whether CorrLog is statistically superior/inferior to the compared method (using paired t-test at 0.05 significance level).}
\label{TabPAS12}
\begin{center}
\begin{tabular}{|c|c|l|l|l|l|l|l|}\hline
\mr{2}{*}{Datasets} & \mr{2}{*}{Methods} & \mc{6}{c|}{Measures} \\\cline{3-8}
 &  & \mc{1}{c|}{Hamming loss} & \mc{1}{c|}{0-1 loss} & \mc{1}{c|}{Accuracy} & \mc{1}{c|}{F1-Score} & \mc{1}{c|}{Macro-F1} & \mc{1}{c|}{Micro-F1} \\\hline
\mr{8}{*}{PASCAL12-PHOW} & CorrLog & 0.070$\pm$0.001 & \textbf{0.790$\pm$0.009} & \textbf{0.344$\pm$0.009} & \textbf{0.393$\pm$0.010} & \textbf{0.369$\pm$0.014} & \textbf{0.449$\pm$0.006} \\
 & ILRs & 0.100$\pm$0.001 $\ast$ & 0.891$\pm$0.009 $\ast$ & 0.269$\pm$0.007 $\ast$ & 0.333$\pm$0.008 $\ast$ & 0.324$\pm$0.008 $\ast$ & 0.370$\pm$0.005 $\ast$ \\
 & IBLR & 0.068$\pm$0.001 $\circledast$ & 0.869$\pm$0.009 $\ast$ & 0.219$\pm$0.005 $\ast$ & 0.252$\pm$0.003 $\ast$ & 0.253$\pm$0.007 $\ast$ & 0.345$\pm$0.005 $\ast$ \\
 & MLkNN & 0.069$\pm$0.001 $\circledast$ & 0.883$\pm$0.008 $\ast$ & 0.191$\pm$0.006 $\ast$ & 0.218$\pm$0.005 $\ast$ & 0.213$\pm$0.007 $\ast$ & 0.306$\pm$0.006 $\ast$ \\
 & CC & 0.097$\pm$0.001 $\ast$ & 0.862$\pm$0.012 $\ast$ & 0.291$\pm$0.010 $\ast$ & 0.350$\pm$0.010 $\ast$ & 0.340$\pm$0.007 $\ast$ & 0.380$\pm$0.006 $\ast$ \\
 & MMOC & \textbf{0.067$\pm$0.001} $\circledast$ & 0.865$\pm$0.003 $\ast$ & 0.227$\pm$0.005 $\ast$ & 0.262$\pm$0.007 $\ast$ & 0.200$\pm$0.007 $\ast$ & 0.346$\pm$0.004 $\ast$ \\
\hline
\mr{8}{*}{PASCAL12-CNN} & CorrLog & 0.040$\pm$0.001 & 0.526$\pm$0.010 & \textbf{0.639$\pm$0.007} & \textbf{0.695$\pm$0.007} & \textbf{0.674$\pm$0.006} & \textbf{0.708$\pm$0.006} \\
 & ILRs & 0.051$\pm$0.001 $\ast$ & 0.613$\pm$0.002 $\ast$ & 0.581$\pm$0.005 $\ast$ & 0.649$\pm$0.006 $\ast$ & 0.638$\pm$0.005 $\ast$ & 0.658$\pm$0.005 $\ast$ \\
 & IBLR & 0.045$\pm$0.001 $\ast$ & 0.574$\pm$0.006 $\ast$ & 0.575$\pm$0.009 $\ast$ & 0.627$\pm$0.010 $\ast$ & 0.613$\pm$0.008 $\ast$ & 0.657$\pm$0.006 $\ast$ \\
 & MLkNN & 0.045$\pm$0.002 $\ast$ & 0.575$\pm$0.012 $\ast$ & 0.566$\pm$0.015 $\ast$ & 0.616$\pm$0.017 $\ast$ & 0.604$\pm$0.011 $\ast$ & 0.645$\pm$0.013 $\ast$ \\
 & CC & 0.055$\pm$0.001 $\ast$ & 0.615$\pm$0.010 $\ast$ & 0.579$\pm$0.009 $\ast$ & 0.647$\pm$0.010 $\ast$ & 0.623$\pm$0.005 $\ast$ & 0.643$\pm$0.007 $\ast$ \\
 & MMOC & \textbf{0.039$\pm$0.001} $\circledast$ & \textbf{0.525$\pm$0.005} & 0.619$\pm$0.006 $\ast$ & 0.669$\pm$0.007 $\ast$ & 0.659$\pm$0.004 $\ast$ & 0.699$\pm$0.005 $\ast$ \\
\hline
\end{tabular}
\end{center}
\end{table*}

\subsection{Complexity Analysis and Execution Time}
Table \ref{TabComplexity} summarizes the algorithm computational complexity of all MLC methods.
The training computational cost of both CorrLog and ILRs are linear to the number of labels, while CorrLog causes more testing computational cost than ILRs due to the iterative belief propagation algorithm.
In contrast, the training complexity of CC and MMOC are polynomial to the number of labels.
The two instance-based methods, MLkNN and IBLR, are relatively computational in both train and test stages due to the involvement of instance-based searching of nearest neighbours.
In particular, training MLkNN requires estimating the prior label distribution from training data which needs the consideration of all $k$ nearest neighbours of all training samples. Testing a given sample in MLkNN consists of finding its $k$-nearest neighbours and applying maximum a posterior (MAP) inference.
Different from MLkNN, IBLR constructs logistic regression models by adopting labels of $k$-nearest neighbours as features.

To evaluate the practical efficiency,
Table \ref{TabTime} presents the execution time (train and test phase) of all comparison algorithms under Matlab environment. A Linux server equipped with Intel Xeon CPU ($8$ cores $@$ $3.4$ GHz) and $32$ GB memory is used for conducting all the experiments.
CorrLog is implemented in Matlab language, while ILRs is implemented based on LIBlinear's mex functions.
MMOC is evaluated using the authors' Matlab code which also builds upon LIBlinear.
As for IBLR, MLkNN and CC, the MEKA Java library is called via a Matlab wrapper.
Based on the comparison results, the following observations can be made:
1) the execution time is largely consistent with the complexity analysis, though there maybe some unavoidable computational differences between Matlab scripts, mex functions and Java codes;
2) CorrLog's train phase is very efficient and its test phase is also comparable with ILRs, CC and MMOC;
3) CorrLog is more efficient than IBLR and MLkNN in both train and test stages.

\begin{table}
\caption{Computational complexity analysis. Recall that $n$ stands for the number of train images, $D$ stands for the dimension of the features, and $m$ stands for the number of labels.
Note that $C$ is the iteration number of the max-product algorithm in CorrLog, and $K$ is the number of nearest neighbours in MLkNN and IBLR.}
\label{TabComplexity}
\begin{center}
\begin{tabular}{|c|c|c|}\hline
Methods & Train & Test per image \\\hline
CorrLog & $\mathcal O(nDm)$ & $\mathcal O(Dm+Cm^2)$ \\
ILRs & $\mathcal O(nDm)$ & $\mathcal O(Dm)$ \\
IBLR & $\mathcal O(Kn^2Dm+nDm)$ & $\mathcal O(KnDm+Dm)$ \\
MLkNN & $\mathcal O(Kn^2Dm)$ & $\mathcal O(KnDm)$ \\
CC & $\mathcal O(nDm+nm^2)$ & $\mathcal O(Dm+m^2)$ \\
MMOC & $\mathcal O(nm^3+nDm^2+n^4)$ & $\mathcal O(m^3)$ \\
\hline
\end{tabular}
\end{center}
\end{table}

\begin{table*}
\caption{Average execution time comparison on all datasets.}
\label{TabTime}
\begin{center}
\resizebox{\textwidth}{!}{
\begin{tabular}{|c|c|c|c|c|c|c|c|c|c|c|c|c|c|c|}\hline
 & \mc{2}{c|}{MULANscene} & \mc{2}{c|}{MITscene-PHOW} & \mc{2}{c|}{MITscene-CNN}
  & \mc{2}{c|}{PASCAL07-PHOW} & \mc{2}{c|}{PASCAL07-CNN} & \mc{2}{c|}{PASCAL12-PHOW} & \mc{2}{c|}{PASCAL12-CNN}\\\cline{2-15}
 & Train & Test & Train & Test & Train & Test & Train & Test & Train & Test & Train & Test & Train & Test\\\hline
CorrLog & 0.09 & 1.74 & 2.80 & 2.12 & 2.46 & 2.08 & 8.94 & 10.68 & 8.35 & 11.08 & 9.67 & 12.58 & 8.62 & 13.06\\
ILRs & 2.54 & 0.02 & 39.50 & 0.37 & 7.50 & 0.15 & 872.77 & 4.79 & 122.73 & 1.56 & 1183.45 & 5.59 & 161.71 & 1.83\\
IBLR & 12.01 & 2.63 & 218.98 & 53.31 & 215.28 & 52.19 & 3132.18 & 779.15 & 2833.94 & 688.53 & 4142.75 & 1034.86 & 3824.06 & 947.90\\
MLkNN & 10.29 & 2.36 & 188.08 & 45.87 & 176.52 & 42.78 & 2507.51 & 628.61 & 2232.19 & 551.26 & 3442.21 & 863.17 & 3020.29 & 779.50\\
CC & 5.48 & 0.06 & 40.71 & 0.55 & 26.64 & 0.65 & 74315.65 & 7.48 & 8746.82 & 8.15 & 137818.99 & 8.38 & 15926.62 & 9.57\\
MMOC & 851.98 & 0.51 & 2952.77 & 0.70 & 2162.13 & 0.48 & 86714.47 & 33.08 & 38403.54 & 17.75 & 97856.16 & 31.43 & 45541.01 & 20.66\\
\hline
\end{tabular}
}
\end{center}
\end{table*}

\section{Conclusion}
We have proposed a
new MLC algorithm CorrLog and applied it to multilabel image classification.
Built upon IRLs, CorrLog explicitly models the
pairwise correlation between labels, and thus improves the
effectiveness for MLC. Besides, by using the elastic net regularization, CorrLog is able to exploit the sparsity in both feature selection and label correlations, and thus further boost the performance of MLC. Theoretically, we have shown that the
generalization error of CorrLog is upper bounded and is independent
of the number of labels. This suggests the generalization bound
holds with high confidence even when the number of labels is large.
Evaluations on four benchmark multilabel image datasets confirm the effectiveness of CorrLog for multilabel image classification and show its competitiveness with the state-of-the-arts.


\bibliography{CorrLog}
\bibliographystyle{IEEEtran}

\end{document}